\documentclass[sigconf,natbib=true,screen,anonymous=false]{acmart}
\usepackage{xspace,balance,tabularx,multirow}
\usepackage{flushend}
\usepackage{xcolor}
\usepackage{tikz}
\usetikzlibrary{plotmarks}
\usepackage{pgfplots}
\pgfplotsset{compat=1.16}
\usetikzlibrary{patterns}
\usepackage{subfig}
\usepackage[ruled, vlined, linesnumbered]{algorithm2e}
\usepackage{colortbl}
\usepackage{bbold}
\SetKwComment{Comment}{$\triangleright$\ }{}
\usepackage{enumitem}
\usepackage{tablefootnote}
\usepackage{upgreek,textgreek}
\usepackage{pifont}% http://ctan.org/pkg/pifont
\usepackage[noabbrev]{cleveref}
\usepackage{titlecaps}
\usepackage{lipsum}
\usepackage{makecell}
\usepackage{subcaption}
\usepackage[ruled, vlined, linesnumbered]{algorithm2e}
\usepackage{booktabs}

\usepackage{listings} 

\lstdefinestyle{prompttemplate}{
    backgroundcolor=\color{gray!5},  
    basicstyle=\small\ttfamily,  
    frame=single,  
    frameround=tttt,  
    breaklines=true,  
    columns=fullflexible, 
    numbers=none,  
    upquote=true  
}

\lstdefinestyle{datasetstyle}{
    backgroundcolor=\color{blue!3},  
    basicstyle=\small\ttfamily,  
    frame=single,
    frameround=tttt,
    breaklines=true,
    columns=fullflexible,
    numbers=none
}

\lstdefinestyle{keywordstyle}{
    backgroundcolor=\color{green!5},  
    basicstyle=\small\ttfamily,  
    frame=single,  
    frameround=tttt,  
    breaklines=true,  
    columns=fullflexible, 
    numbers=none,  
    upquote=true  
}

\lstdefinestyle{summarystyle}{
    backgroundcolor=\color{red!5},  
    basicstyle=\small\ttfamily,  
    frame=single,
    frameround=tttt,
    breaklines=true,
    columns=fullflexible,
    numbers=none
}

\usepackage{amsmath}

\pgfplotsset{every tick label/.append style={font=\tiny}}

\newlength{\starsize}
\newlength{\starspread}
\tikzset{starsize/.code={\setlength{\starsize}{#1}},
         starspread/.code={\setlength{\starspread}{#1}}}
\tikzset{starsize=1mm,
         starspread=3mm}
\pgfdeclarepatternformonly[\starspread,\starsize]% variables
  {my fivepointed stars}% name
  {\pgfpointorigin}% lower left corner
  {\pgfqpoint{\starspread}{\starspread}}% upper right corner
  {\pgfqpoint{\starspread}{\starspread}}% tilesize
  {% shape description
   \pgftransformshift{\pgfqpoint{\starsize}{\starsize}}
   \pgfpathmoveto{\pgfqpointpolar{18}{\starsize}}
   \pgfpathlineto{\pgfqpointpolar{162}{\starsize}}
   \pgfpathlineto{\pgfqpointpolar{306}{\starsize}}
   \pgfpathlineto{\pgfqpointpolar{90}{\starsize}}
   \pgfpathlineto{\pgfqpointpolar{234}{\starsize}}
   \pgfpathclose%
   \pgfusepath{fill}
  }
  
\AtBeginDocument{%
  }

\acmSubmissionID{413}

\newcommand{\argmax}[1]{\underset{#1}{\operatorname{arg}\,\operatorname{max}}\;}

\makeatletter
\newcommand*\bigcdot{\mathpalette\bigcdot@{.5}}
\newcommand*\bigcdot@[2]{\mathbin{\vcenter{\hbox{\scalebox{#2}{$\m@th#1\bullet$}}}}}
\makeatother

\newcommand{\stitle}[1]{\vspace*{0.5em}\noindent{\underline{\bf #1.\/}}}

\newcommand{\V}{\mathcal{V}\xspace}
\newcommand{\G}{\mathcal{G}\xspace}

\newcommand{\N}{\mathcal{N}\xspace}
\newcommand{\EDG}{\mathcal{E}\xspace}

\newcommand{\C}{\mathcal{C}\xspace}

\newcommand{\WM}{\mathbf{W}\xspace}
\newcommand{\AM}{\mathbf{A}\xspace}
\newcommand{\DM}{\mathbf{D}\xspace}
\newcommand{\IM}{\mathbf{I}\xspace}
\newcommand{\SM}{\mathbf{S}\xspace}
\newcommand{\MM}{\mathbf{M}\xspace}
\newcommand{\PM}{\mathbf{P}\xspace}

\newcommand{\YM}{\mathbf{Y}\xspace}
\newcommand{\XM}{\mathbf{X}\xspace}

\newcommand{\HM}{\mathbf{H}\xspace}

\newcommand{\algo}{\texttt{STAD}\xspace}
\newcommand{\GCond}{\texttt{GCond}\xspace}
\newcommand{\GCDM}{\texttt{GCDM}\xspace}
\newcommand{\TCond}{\texttt{TCond}\xspace}
\newcommand{\TCDM}{\texttt{TCDM}\xspace}
\newcommand{\GDEM}{\texttt{GDEM}\xspace}
\newcommand{\ClustGDD}{\texttt{ClustGDD}\xspace}
\newcommand{\SGDD}{\texttt{SGDD}\xspace}
\newcommand{\ClustTDD}{\texttt{ClustTDD}\xspace}
\newcommand{\CGC}{\texttt{CGC}\xspace}
\newcommand{\SFGC}{\texttt{SFGC}\xspace}
\newcommand{\GCSR}{\texttt{GCSR}\xspace}
\newcommand{\TDD}{\texttt{TDD}\xspace}
\newcommand{\DDforTC}{\texttt{DD4TC}\xspace}
\newcommand{\ASD}{\texttt{ASD}\xspace}
\newcommand{\DaLLME}{\texttt{DaLLME}\xspace}
\newcommand{\DiLM}{\texttt{DiLM}\xspace}

\newcommand{\GNNLLM}{\texttt{GNN-LLM}\xspace}
\newcommand{\TAPE}{\texttt{TAPE}\xspace}
\newcommand{\ENGINE}{\texttt{ENGINE}\xspace}
\newcommand{\LLaGA}{\texttt{LLaGA}\xspace}
\newcommand{\CTGL}{\texttt{CTGL}\xspace}

\newcommand{\eat}[1]{}

\def\addlegendimage{\csname pgfplots@addlegendimage\endcsname}

\makeatletter
\newcommand\footnoteref[1]{\protected@xdef\@thefnmark{\ref{#1}}\@footnotemark}
\makeatother

\let\oldnl\nl% Store \nl in \oldnl
\newcommand{\nonl}{\renewcommand{\nl}{\let\nl\oldnl}}% Remove line number for one line

\definecolor{myred}{HTML}{fd7f6f}
\definecolor{myred_new}{HTML}{D8D8D8}
\definecolor{myred_new2}{HTML}{D7191C}
\definecolor{myblue}{HTML}{7eb0d5}
\definecolor{mygreen}{HTML}{b2e061}
\definecolor{mypurple}{HTML}{bd7ebe}
\definecolor{myorange}{HTML}{ffb55a}
\definecolor{myyellow}{HTML}{ffee65}
\definecolor{mypurple2}{HTML}{beb9db}
\definecolor{mypink}{HTML}{fdcce5}
\definecolor{mycyan}{HTML}{8bd3c7}

\definecolor{myblue2}{HTML}{115f9a}
\definecolor{myred2}{HTML}{c23728}

\definecolor{myorange-new}{HTML}{E56400}
\definecolor{mygreen-new}{HTML}{00788B}
\definecolor{greenacc}{HTML}{2ca02c}
\definecolor{Blue}{RGB}{173, 216, 230} % 
\definecolor{DarkBlue}{RGB}{0, 191, 255}
\begin{document}

\title{Semi-Supervised Text-Attributed Graph Distillation}

\author{Yurui Lai}
\affiliation{%
  \institution{Hong Kong Baptist University}
  \city{Hong Kong}
  \country{China}
}
\email{csyrlai@comp.hkbu.edu.hk}

\author{Samir Moustafa}
\affiliation{%
  \institution{CeMM Research Center for Molecular Medicine of the Austrian Academy of Sciences\\
  Universität Vienna}
  \city{Vienna}
  \country{Austria}}
\email{samir.moustafa@univie.ac.at}

\author{Renchi Yang$^\dagger$}
\affiliation{%
  \institution{Hong Kong Baptist University}
  \city{Hong Kong}
  \country{China}
}
\email{renchi@hkbu.edu.hk}

\author{Tsz Nam Chan}
\affiliation{%
 \institution{Shenzhen University}
 \city{Shenzhen}
 \country{China}}
\email{edisonchan@szu.edu.cn}

\renewcommand{\shortauthors}{Yurui Lai, Samir Moustafa, Renchi Yang, and Tsz Nam Chan.}

\begin{abstract}
{\em Text-Attributed Graphs} (TAGs) have emerged as an expressive data model for integrating graph topology with rich textual semantics. Existing representation learning methods over TAGs suffer from severe scalability bottlenecks, particularly together with {\em Large Language Models} (LLMs). 
While data distillation offers a promising data-centric solution, existing methods fail to capture the complex interplay between graph and text modalities, struggle with the label scarcity inherent in semi-supervised settings, and lack the ability to produce the human-readable textual attributes required for downstream LLM-based tasks. 

To address these challenges, we propose \algo{}, a unified semi-supervised framework guided by the {\em Wasserstein Distance} (WSD). 
Grounded in our empirical findings on real TAGs, \algo{} introduces a graph-text collaborative encoding module that utilizes dual-pathway encoders (graph-aware and -free) within a collaborative self-training scheme to harvest reliable pseudo-labels and fuse complementary graph-text features. 
Furthermore, we develop a theoretically grounded WSD-based graph sketching algorithm and a cost-effective LLM text synthesis module, which leverages cluster-based keyword extraction to generate coherent, human-readable summaries for condensed nodes. 
Extensive experiments on benchmark datasets demonstrate that \algo{} achieves a state-of-the-art performance-compression trade-off in terms of both GNN- and LLM-based downstream tasks, enabling effective and efficient TAG learning or analytics.

% Text-Attributed Graphs (TAGs) combines graph structure with semantic text to overcome traditional attributed graph learning limitations. Hybrid GNN-LM/LLM methods have advanced TAG learning.
% However, TAG learning faces scalability bottlenecks. Large real-world TAGs and high LLM integration costs hinder deployment, and single-modality distillation methods are inapplicable due to cross-modality, label scarcity and text readability issues.
% To address these, this paper proposes \textsf{TAGDD} (Text-Attributed Graphs Data Distillation), a unified framework combining self-co-training graph condensation and LLM-based text synthesis. It leverages Wasserstein Distance (WSD) to quantify cross-modality distribution shifts during distillation; after dual-branch GNN self-training, a novel clustering algorithm is adopted for graph condensation, effectively alleviating label scarcity and bimodality fusion issues. Additionally, it uses LLM-driven keyword summarization to generate readable synthetic text for the condensed TAGs.
% Experiments on benchmark TAG datasets demonstrate that TAGDD achieves a state-of-the-art performance-compression trade-off against existing distillation baselines, with an average SOTA advance of [X.X\%], while significantly reducing the training costs of LLM-augmented TAG models and maintaining competitive downstream performance to enable large-scale real-world deployment.
% This work pioneers systematic TAG distillation, providing a reliable metric and efficient framework for large-scale TAG learning.
% \renchi{rewrite after refining methodology}
\end{abstract}



% %% A "teaser" image appears between the author and affiliation
% %% information and the body of the document, and typically spans the


\maketitle
\footnotetext{$\dagger$ Corresponding author}

\newcommand\kddavailabilityurl{https://github.com/laiyurui/STAD-Semi-Supervised-Text-Attributed-Graph-Distillation}
\ifdefempty{\kddavailabilityurl}{}{
\begingroup\small\noindent\raggedright\textbf{Resource Availability:}\\
% please change the following context to include multiple artifacts if necessary, including data, models, code, etc.
The source code of this paper has been made publicly available at \url{\kddavailabilityurl}.
\endgroup
}

\section{Introduction}

{\em Text-Attributed Graphs} (TAGs) have emerged as a popular graph data paradigm, integrating rich semantic text with complex relations among textual entities~\cite{yan2023comprehensive, feng2024taglas, chen2024text, zhang2024dtgb}. By bridging the gap between graph data and natural language, TAGs unlock the potential of {\em Large Language Models} (LLMs) for diverse cross-domain applications~\cite{li2023survey, liu2023towards,jin2024large, wang2025can}, from drug discovery~\cite{li2025druglm} to fraud detection~\cite{yang2025flag}, etc.
Through exploiting the synergy between two modalities in TAGs, recent hybrid frameworks combining {\em Graph Neural Networks} (GNNs)~\cite{kipf2016semi, velivckovic2017graph,hamilton2017inductive} with LLMs have achieved remarkable results for learning tasks over TAGs~\cite{zhao2022learning, he2023harnessing, zhang2024text, chen2024text, zhu2024efficient, tang2024graphgpt, chen2024llaga,sun2025large,wullms}. 
Despite the advancements made, these approaches still suffer from a critical scalability bottleneck in practical deployment, as real-world TAGs often scale to hundreds of thousands of nodes and edges~\cite{chen2024text}. 
The integration with LLMs further significantly increases the computational and financial overhead demanded by these methods, which, in turn, severely limits their applicability in large-scale industrial and scientific scenarios.

{\em Data distillation} offers a promising data-centric solution to this scalability challenge in large-scale learning tasks~\cite{sachdeva2023data,lei2023comprehensive}.
% It aims to synthesize a condensed yet informative dataset from the original large-scale data, such that models trained on the synthetic dataset can achieve performance comparable to those trained on the original data. 
Its goal is to synthesize a condensed, small-scale dataset that preserves the training dynamics and information of the original large-scale data.
In recent years, substantial efforts~\cite{jingraph, yang2023does, liu2024graph, lai2025simple} have been made towards {\em graph distillation} (a.k.a. graph condensation) and {\em text distillation}~\cite{li2021data, sucholutsky2021soft, maekawa-etal-2023-dataset, tao2024textual}, where the former seeks to condense the original graphs with nodal attributes into small ones, while the latter focuses on compressing text into concise semantic representations.
However, extending these successes to TAGs still presents unique, non-trivial challenges.
Firstly, the naive combination of graph and text distillation fails to capture the complex interplay between graph and textual modalities~\cite{jingraph, zhou2025text}. Graph distillation methods typically treat text as static numerical features, failing to maintain semantic consistency, while text distillation overlooks structural features.
Second, existing graph condensation methods simply construct continuous features that lack interpretability as node attributes. But for TAGs, it is necessary to generate human-readable textual attributes such that the distilled TAGs remain compatible with downstream LLM-based tasks, e.g., prompt-based learning or explanation generation~\cite{he2023harnessing}.
On top of that, real-world TAGs often present severe label scarcity, where standard distillation techniques that often rely on abundant supervision to align gradients or distributions, falter or even fail.

To overcome the aforementioned challenges, this paper proposes \algo{} (\underline{S}emi-supervised \underline{T}ext-\underline{A}ttributed Graph \underline{D}istillation), a unified framework designed to distill TAGs into compact, human-readable graphs under semi-supervised settings.
Our approach is grounded in our quantitative empirical studies, where we identify that (i) graph-aware models (GNNs) and graph-free models (MLPs) capture complementary information, and simply optimizing one modality is insufficient for effective condensation of TAGs, and (ii) the {\em Wasserstein Distance} (WSD)~\cite{conv-wsd, yang2023does} between the distributions of original and condensed graphs is strongly correlated with downstream performance.
% Guided by this metric, TAGDD introduces two novel components.
Inspired by these observations, we first propose a {\em graph-text collaborative encoding} module, which employs {\em dual-pathway} encoders (graph-aware and -free pathways) within a collaborative self-training scheme as complementary learners, where pseudo-supervision is iteratively generated for better feature fusion, thereby mitigating label scarcity and cross-modality misalignment.
Subsequently, building on these fused features, a theoretically-grounded graph sketching aiming at minimizing the WSD between distribution shifts between original and condensed TAGs is leveraged for condensing both graph topology, attributes, and labels.
Furthermore, we construct textual attributes for distilled graphs through our {\em Keywords-based LLM text synthesis}.
Instead of optimizing abstract embeddings or directly harnessing LLMs for text summarization, which is powerful but costly~~\cite{zhang2025comprehensive,zhang2025systematic}, \algo{} resorts to a three-stage pipeline, where keywords are first extracted from condensed clusters, and LLMs are then cost-effectively harnessed to generate coherent candidate text from these keywords. 
Lastly, a WSD-validation-guided selection strategy is adopted to identify the synthetic texts that best preserve the distributional characteristics of the original data.

In sum, our major contributions in this paper are as follows:
\begin{itemize}[leftmargin=*]
\item We pioneer the systematic exploration of TAG distillation. We propose the use of WSD as a reliable metric to quantify distribution shifts in TAGs, offering a theoretical basis for assessing distillation quality in graph-text contexts.

\item We develop \algo{} framework, which integrates graph-text collaborative encoding to resolve label scarcity and modality misalignment and modality fusion, and WSD-guided graph sketching and keywords-based LLM text synthesis to create high-quality condensed graphs and textual attributes.

\item Our extensive experiments on multiple benchmark TAG datasets demonstrate that \algo{} consistently achieves a state-of-the-art performance-compression trade-off in both GNN- and LLM-based node classification tasks, compared to existing graph or text distillation solutions.
\end{itemize}

\section{Related Work}
\stitle{Graph Distillation}
Graph Distillation (or Graph Condensation) is a data-centric solution to reduce the computational and storage costs of GNNs on large-scale graphs, whose goal is to synthesize compact graphs, maintaining the key information of original graphs for competitive model performance. \GCond~\cite{jingraph} is a simple gradient-matching based method which synthesizes graphs by matching the GNN's gradients on the synthesis and original graphs. \SGDD~\cite{yang2023does} proposes a structure-broadcasting scheme to preserve the original structural information in the gradient matching process. \GCDM~\cite{liu2022graph} and \GDEM~\cite{liu2024graph} are distribution matching-based methods, with \GCDM aligning the receptive field distributions of synthetic and original graphs via MMD, and GDEM focusing on matching the eigenbasis and spectral distributions of the two graphs. \SFGC~\cite{zheng2023structure} and \GCSR~\cite{liu2024graph} are trajectory matching-based distillation methods, which align the long-term learning dynamics of GNNs on synthetic and original graphs. \CGC~\cite{gao2025rethinking} and \ClustGDD~\cite{lai2025simple} are clustering-based methods, with \CGC conducting class partition to derive condensed node features and topology in closed form, and \ClustGDD optimizes FID between synthetic and original graphs via clustering. 

\stitle{Text Distillation}
Text distillation is the process of synthesizing compact text datasets for NLP tasks. \TDD~\cite{sucholutsky2021soft} updates network parameters via gradient descent with distilled samples and soft labels, computes loss on real text, and optimizes the distilled component by loss gradients to capture text semantics, then decodes distilled text using the original word embedding dictionary. \DDforTC~\cite{li2021data} uses the gradients from training on the original texts to be back-propagated to the distillation texts. \ASD~\cite{maekawa-etal-2023-dataset} simultaneously optimizes distilled input embeddings, soft labels, and attention labels using gradient descent, and guides model parameter updates by combining task loss and transformer attention loss. \DiLM~\cite{maekawa2025dilm} uses a GPT-2 model to generate candidate distilled texts, and learn a generation probability via matching the learner's gradients on both original and distilled texts. \DaLLME~\cite{tao2024textual} conducts k-centroid on text embeddings and translates the clustered embedding center back to readable text via a T5~\cite{raffel2020exploring} model finetuned on embedding-text pairs constructed from the original dataset. 

Although many distillation methods exist for graph and text, no dedicated scheme has been designed for TAGs. Single-modal distillation cannot fully exploit the other modality, motivating us to propose a collaborative graph-text distillation scheme for TAGs.

\stitle{TAG Representation Learning}
LMs (including LLMs), combined with traditional GNNs, have leveraged their abilities on TAG representation learning based on rich semantic information of raw texts in TAGs. LMs can serve as TAG encoders by integrating graph topology and texts into the embedding space~\cite{yan2023comprehensive,chen2024text,zhu2024efficient}. \ENGINE~\cite{zhu2024efficient} combines the inter-layer embeddings of LLMs for text encoding. LMs can serve as TAG augmenters, e.g., \TAPE~\cite{he2023harnessing} makes LLMs to directly generate predictions and explanations to enrich the text of nodes. \CTGL~\cite{zhou2025text} quantitatively compares the complementarity between graph structure and textual attributes, serially augments and encodes TAG using GNNs and LMs, and fuses graph structure and text in a contrastive way. LMs can be the predictors to give final predictions of TAG learning tasks~\cite{tang2024graphgpt, chen2024llaga}. \LLaGA\cite{chen2024llaga} puts the serialized graph structure and node texts into an LLM to generate predictions and natural language descriptions. \cite{wullms} gives an extensive survey of different types of LLMs' performance on TAG node classification. However, these methods will suffer from training efficiency problems when the TAGs scale up, which motivates the exploration of TAG distillation. We
In this work, we define LLM4TAG as large language models tailored for TAG learning tasks~\cite{su2025large}, and use this standardized terminology consistently in the experimental section.

\section{Preliminaries}

\subsection{Notations and Problem Formulation}
\label{sec:notations}
\stitle{Notations} Let $\G=(\V,\EDG,\mathcal{S})$ be a {\em Text-Attributed Graph} (TAG), where $\V$ denotes the set of $N$ nodes ($N=|\V|$), $\EDG$ denotes the set of $M$ edges ($M=|\EDG|$), $\mathcal{S} = \{s_i\}_{v_i \in \V}$ is the collection of text sequences for nodes. 
We denote the neighbor set of $v_i$ as $\N(v_i)$ with degree $d(v_i)=|\N(v_i)|$. $\AM\in\{0,1\}^{N\times N}$ is used to represent the adjacency matrix of $\G$ (where $\AM_{i,j}=1$ if $(v_i, v_j)\in\EDG$, otherwise $\AM_{i,j}=0$) and $\DM\in\mathbb{R}^{N\times N}$ symbolizes the diagonal degree matrix. 
The normalized adjacency matrix is $\tilde{\AM}=\DM^{-\frac{1}{2}}\AM\DM^{-\frac{1}{2}} $. 
Let $\XM\in\mathbb{R}^{N\times D}$ be the text embeddings of nodes, wherein $\XM_i$ (the $i$-th row) denotes the vectorized encoding of $s_i$ for node $v_i\in \V$ via a text encoder $g:\mathcal{S}\to\XM$. By default, we adopt a pretrained Sentence-BERT (SBERT)~\cite{reimers2019sentence} as $g$ throughout this paper.
Let $\mathcal{Y}$ be the task-specific class set, and $K=|\mathcal{Y}|$ is the number of node classes.
% Let $K$ be the number of node classes.
We denote by $\YM\in \{0,1\}^{N\times K}$ the ground-truth label matrix of $\G$
% and $\mathcal{Y}=\{(v_i,y_i)\}_{i=1}^N$ be the node labels
% class labels of nodes in $\G$, 
where $\YM_{i,k}=1$ if $y_k\in \mathcal{Y}$ ($1\le k\le K$) is $v_i$'s ground-truth class label and $0$ otherwise. 
Under semi-supervised settings, only a small set of nodes $\V_\text{tr}\subseteq \mathcal{V}$ ($|\V_\text{tr}|\ll N$) are labeled, referred to as training samples. 

\stitle{Problem Formulation}
The overarching goal of {\em Text-Attributed Graph Distillation} (TAGD) is to distill a compact, high-fidelity compressed TAG $\G' = (\V', \EDG', \mathcal{S}')$ from the original larger $\G$, where $\V^\prime$
($N^\prime=|\V^\prime| \ll N$) and $\EDG^\prime$ ($M^\prime=|\EDG^\prime| \ll M$) stand for the node and edge sets of $\G^\prime$, respectively, and $\mathcal{S}^\prime = \{s^\prime_i\}_{v^\prime_i \in \V^\prime}$ represents human-readable text sequences for condensed nodes therein.
Let $\AM'$ and $\YM'$ be the adjacency matrix and node labels of $\G^\prime$, respectively.
Formally, the objective of the TAGD problem can be formulated as:
\begin{equation}\label{eq:TAGD-obj}
\begin{gathered}
\min_{\G'} \, {f}\left(\mathcal{M}_{\boldsymbol{\phi}_{\G'}}(\AM, \mathcal{S}), \YM\right) \\
\text{s.t.} \, \boldsymbol{\phi}_{\G'} = \arg\min_{\boldsymbol{\phi}} f\left(
            \mathcal{M}_{\boldsymbol{\phi}}(\AM', \mathcal{S}'), \YM'
        \right), \\
\end{gathered}
\end{equation}
where $\mathcal{M}_{\boldsymbol{\phi}}$ denotes a general model parameterized by $\boldsymbol{\phi}$, and $f(\cdot)$ is the task-specific criteria. 
In this paper, we consider the semi-supervised node classification setting and the {\em distillation ratio} is defined as $r=\frac{N^\prime}{|\V_\text{tr}|}\leq 1$.

\subsection{2-Wasserstein Distance (WSD)}
The \textit{2-Wasserstein Distance} (WSD) in \textit{optimal transport}~\cite{solomon2015convolutional} is a classic distance function used to measure shift between two probability distributions $\mathcal{P}$ and $\mathcal{Q}$:
  \begin{equation*}
  \textsf{WSD}(\mathcal{P}, \mathcal{Q}) = \inf_{\gamma \in \Gamma(\mathcal{P}, \mathcal{Q})} \int_{\mathcal{X} \times \mathcal{X}} \|\mathbf{p} - \mathbf{q}\| \, \text{d}\gamma(\mathbf{p}, \mathbf{q})
  \end{equation*}
where $ \Gamma(\mathcal{P}, \mathcal{Q}) $ is the set of joint distributions with marginals $ \mathcal{P} $ and $ \mathcal{Q} $, $ \mathcal{X} $ is the feature space, and $ \|\mathbf{p} - \mathbf{q}\| $ is the Euclidean distance. \SGDD~\cite{yang2023does} uses WSD as an alternative to the Laplacian Energy Distribution shift coefficient for graph structure optimization. \ClustGDD~\cite{lai2025simple} finds a correlation between {\em Fréchet Inception Distance} (FID) and accuracy of GNNs on original and synthetic graphs, where FID between two multivariate Gaussian distributions \(N(\mu_\mathcal{P}, \Sigma_\mathcal{P})\) and \(N(\mu_\mathcal{Q}, \Sigma_\mathcal{Q})\), given by
\[
\text{FID}(\mathcal{P},\mathcal{Q}) = \|\mu_\mathcal{P} - \mu_\mathcal{Q}\|_2^2 + \operatorname{Tr}\big(\Sigma_\mathcal{P} + \Sigma_\mathcal{Q} - 2(\Sigma_\mathcal{P} \Sigma_\mathcal{Q})^{1/2}\big),
\]
corresponds to a special case of WSD when the underlying distributions are multivariate Gaussians. Next, we will compute WSD between original and synthetic textual attributes, graph structures under various distillation ratios, and comprehensively investigate the correlations between WSD and the model performance.

\begin{figure}[tb]  
  \centering 
  \includegraphics[width=\columnwidth]{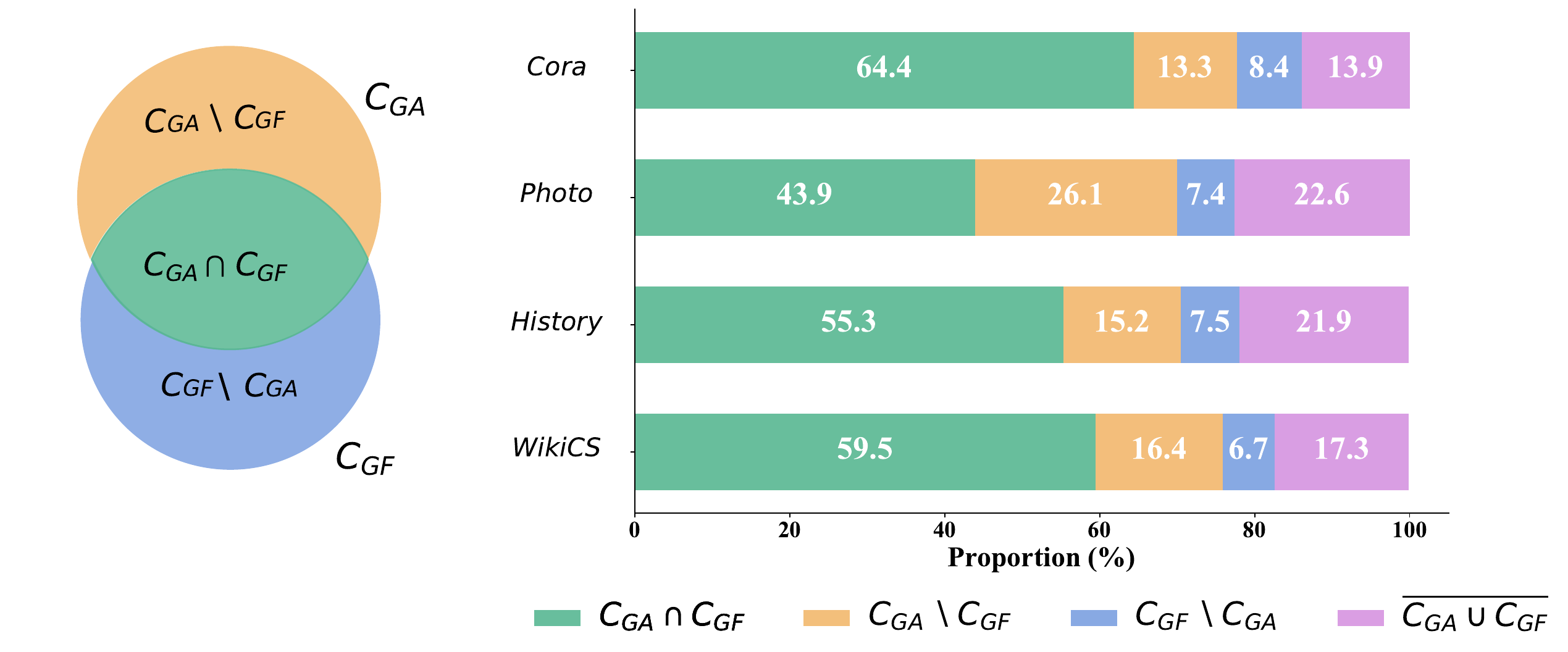} 
  \vspace{-4ex}
  \caption{Complementary coverage of GA vs.\ GF on TAGs.}
  \label{fig:single_col}  
  \vspace{-2ex}
\end{figure}

\begin{figure*}[tb]  
  \centering 
  \includegraphics[width=0.9\linewidth]{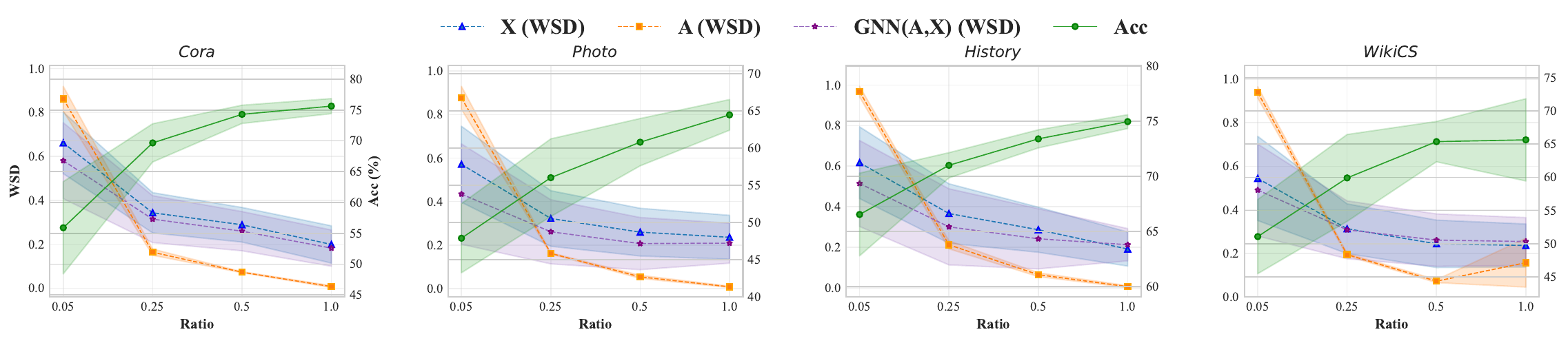} 
  \vspace{-3ex}
  \caption{WSD and node classification accuracy on condensed TAGs vs.\ distillation ratio.}
  \label{fig:acc-wsd}
  \vspace{-2ex}
\end{figure*}

\subsection{Preliminary Studies}
\label{sec:observations_objectives}

We empirically investigate the TAGD task over four real datasets (\textit{Cora}, \textit{Photo}, \textit{History}, and \textit{WikiCS}) pertaining to two aspects before entering into the algorithmic design of our proposed \algo{}.

\stitle{\bf Graph-Aware vs. Graph-Free}\label{sec:GA-GF}
Since both graph structures and textual attributes encode rich semantics in TAGs, a natural question is their roles in the downstream task, i.e., semi-supervised node classification. 
To this end, we compare the {\em graph-aware} (GA) model, i.e., GNNs using  $\XM$ and $\AM$ for node classification, against the {\em graph-free} (GF) model, i.e., MLPs using pure text attribute of $\mathcal{S}$.
Let $\mathcal{C}_\text{GA}$ and $\mathcal{C}_\text{GF}$ be the sets of node samples whose labels are correctly predicted by the GA and GF models, respectively. Accordingly, $\overline{\mathcal{C}_\text{GA}\cup \mathcal{C}_\text{GF}}$ stands for the complementary set containing node samples whose labels cannot be correctly predicted by both models.

Fig.~\ref{fig:single_col} reports the proportions of different sets on the four tested datasets. It can be observed that although $\mathcal{C}_\text{GA}$ and $\mathcal{C}_\text{GF}$ have a big overlap ($43.9\%$-$64.4\%$), their differences are still considerable.
Firstly, this observation underscores the indispensability of both structural and textual information and their complementary strengths in TAGs. Additionally, the $\mathcal{C}_\text{GF}\setminus\mathcal{C}_\text{GA}$ part implies that using GNNs with $\XM$ in the GA model fails to fully capture the textual semantics in $\mathcal{S}$, necessitating new feature encoders.

\stitle{\bf Relation between the WSD and Accuracy}\label{sec:WSD-acc}
Next, we empirically analyze the distributional shifts between attribute matrices of the original TAG $\G$ and the condensed TAG $\G^\prime$, their topological structures (i.e., adjacency matrices), their fused representations via GNNs, as well as their relations to node classification performance. 
Fig.~\ref{fig:acc-wsd} displays the WSD values between $\XM$ and $\XM^\prime$, $\AM$ and ${\AM}^\prime$, $\textsf{GNN}(\AM,\XM)$ and $\textsf{GNN}({\AM}^\prime,\XM^\prime)$, and node classification accuracies (Acc) when varying the distillation ratios, i.e., $r\in [0.05,0.25,0.5,1.0]$ on the four real datasets ({\em Cora}, {\em Photo}, {\em History}, and {\em WikiCS}).

The empirical results reveal a key trend. $\textsf{WSD}({\XM}, {\XM}^\prime)$, $\textsf{WSD}({\AM}, {\AM}^\prime)$, and $\textsf{WSD}(\textsf{GNN}(\AM,\XM),\textsf{GNN}({\AM}^\prime,\XM^\prime))$, increase constantly as $r$ is decreased (i.e., their distributional shifts enlarge), while the classification performance keeps going down.
This observation suggests that maximizing the node classification performance over the TAG $\G$ in Eq.~\eqref{eq:TAGD-obj} can be equivalently transformed into the problem of reducing the discrepancy between $\G$ and $\G^\prime$ in terms of both graph structures and attributes, e.g., minimizing $\textsf{WSD}(\XM,\XM^\prime)$ and $\textsf{WSD}(\AM,\AM^\prime)$ jointly. To reduce bias, we compute the mean and std of the results from different baseline distillation methods.

\section{Methodology}
This section presents our \algo{} method for TAGD under semi-supervised settings.
As illustrated in Fig.~\ref{fig:method_overview}, \algo{} includes three key modules for feature encoding in TAGs, graph sketching, and synthesis of textual attributes.
We begin with elucidating our {\em graph-text collaborative encoding} model for generating node feature vectors $\HM$ in \S\ref{sec:feat-encode}, followed by constructing the sketching matrix $\SM$ for graph condensation in \S\ref{sec:sketch}. \S\ref{sec:text-synthesis} delineates our cost-effective approach for synthesizing textual attributes of nodes in $\G^\prime$ with LLMs.
For the interest of space, we defer the complete algorithm and provide the cost analysis to Appendix~\ref{sec:algorithm_complexity}.

\subsection{Graph-Text Collaborative Encoding}\label{sec:feat-encode}

As per our analysis in \S\ref{sec:GA-GF}, we propose to adopt {\em dual-pathway encoders} and leverage {\em collaborative self-training} to unearth the complementary information underlying the graph and text features for enhanced feature encoding.

\stitle{\bf Dual-Pathway Encoders}
Since the GNN encoder will lose distinct features underlying $\mathcal{S}$, as revealed in \S\ref{sec:GA-GF}, \algo{} includes two paths of encoders: a GA (or graph) encoder via GCNs, and a GF (or text) encoder via an MLP layer.
Specifically,
\begin{equation}
\HM^{\text{\tiny GA}}=\textsf{GCN}(\AM, \XM),\ \HM^{\text{\tiny GF}}=\textsf{MLP}(\XM),
\end{equation}
whose $i$-th rows $\HM_i^{\text{\tiny GA}}\in \mathbb{R}^{h}$ and $\HM_i^{\text{\tiny GF}} \in \mathbb{R}^{h}$ signify the feature vectors of node $v_i$ output by the dual-pathway encoders.
Subsequently, \algo{} calculates attention weights by
\begin{equation}
\alpha^{\text{\tiny GA}}_i = \sigma\left(\WM_{\text{att}} \cdot \HM_i^{\text{\tiny GA}} + \mathbf{b}_{\text{att}}\right),\ \alpha^{\text{\tiny GF}}_i = 1-\alpha^{\text{\tiny GA}}_i,
\end{equation}
where $\WM_{\text{att}} \in \mathbb{R}^{h}$ and $\mathbf{b}_{\text{att}} \in \mathbb{R}$ are learnable parameters, and $\sigma(\cdot)$ denotes the sigmoid function. 
Accordingly, the fused predicted {\em label distributions} (or soft labels) of each node $v_i$ can be derived via a weighted element-wise summation ($\odot$ is the element-wise multiplication) of the feature vectors from both encoders:
\begin{equation}\label{eq:P-def}
    \PM_i = \textsf{softmax}(\alpha^{\text{\tiny GA}}_i \odot \HM_i^{\text{\tiny GA}} + \alpha^{\text{\tiny GF}}_i \odot \HM_i^{\text{\tiny GF}}).
\end{equation}
Inspired by the adaptive mechanism in \cite{Bo_Wang_Shi_Shen_2021, NEURIPS2022_092359ce}, \algo{} learns attention weights to automatically fuse graph and text features. In particular, the soft labels derived from the GA and GF encoders can be represented as follows:
\begin{equation*}
\PM^{\text{\tiny GA}}_i = \textsf{softmax}(\HM_i^{\text{\tiny GA}}),\ \PM^{\text{\tiny GF}}_i = \textsf{softmax}(\HM_i^{\text{\tiny GF}}).
\end{equation*}

\begin{figure*}[tb]  
  \centering 
  \includegraphics[width=0.95\linewidth]{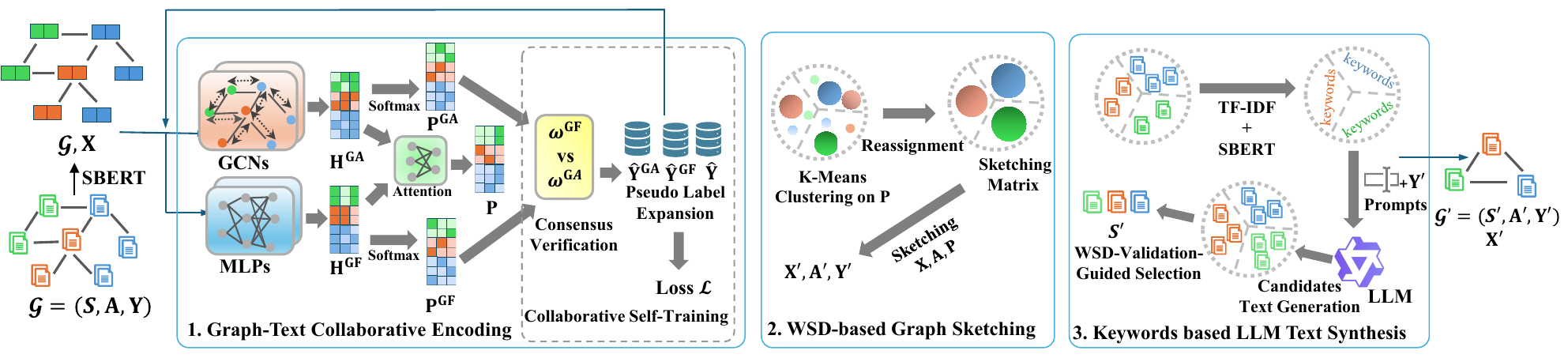} 
  \vspace{-2ex}
  \caption{The Overview of \algo{}}
  \label{fig:method_overview}  
  \vspace{-1ex}
\end{figure*}

\stitle{\bf Collaborative Self-Training} 
To fully leverage scarce labeled data, we introduce \textit{collaborative self-training} (CoST) that treats GA and GF encoders as complementary learners, as they capture distinct features predictive of different labels.
Let $\hat{\YM}$ be the label matrix consisting of both true labels and hard pseudo-labels obtained based on label distribution $\PM$, and $\hat{\YM}^{\text{\tiny GA}}$ and $\hat{\YM}^{\text{\tiny GF}}$ are counterparts from $\PM^{\text{\tiny GA}}$ and $\PM^{\text{\tiny GF}}$, respectively.
CoST trains the model with such predictions and pseudo-labels by jointly optimizing the following composite loss:
\begin{equation}
\begin{aligned}
\mathcal{L} &= \textsf{CE}( \PM, \hat{\YM}) + \textsf{CE}(\PM^{\text{\tiny GF}},\hat{\YM}^{\text{\tiny GF} } )+ \textsf{CE}(\PM^{\text{\tiny GA}},\hat{\YM}^{\text{\tiny GA}}),
\end{aligned}
\end{equation}
where $\textsf{CE}(\cdot,\cdot)$ denotes the cross-entropy loss.

Next, we elaborate on the constructions of $\hat{\YM}$, $\hat{\YM}^{\text{\tiny GA}}$, and $\hat{\YM}^{\text{\tiny GF}}$.
Unlike conventional self-training~\cite{lee2013pseudo, li2018deeper}, where the pseudo-labels from a single model are prone to error accumulation, CoST employs a {\em consensus verification} between GA and GF modules to filter uncertain pseudo-labels.
For ease of exposition, we represent by $\mathcal{U}$, $\mathcal{U}^{\text{\tiny GA}}$, and $\mathcal{U}^\text{\tiny GF}$ the labeled node sets of $\hat{\YM}$, $\hat{\YM}^{\text{\tiny GA}}$, and $\hat{\YM}^{\text{\tiny GF}}$, respectively.
Initially, $\mathcal{U}=\mathcal{U}^{\text{\tiny GA}}=\mathcal{U}^{\text{\tiny GF}}=\V_\text{tr}$.
We iteratively expand these labeled node sets by identifying pseudo-labeled data from $\V\setminus \mathcal{U}$.

Specifically, in each epoch, we obtain a set $\mathcal{R}$ of candidate nodes for pseudo-labeling based on the GA pathway $\PM^{\text{\tiny GA}}$ by selecting the $\rho$ proportion of unlabeled nodes with the highest {\em confidence scores} where the confidence score of $v_i$ is denoted by $\max_{1\le k\le K}\PM^{\text{\tiny GA}}_{i,k}$.
Notice that the value of $\rho$ is incremented by 0.1 per epoch until reaching a maximum of 1.0 as in~\cite{karisani2023neural}.
Subsequently, the candidate set $\mathcal{R}$ is partitioned into a {\em consensus} set $\mathcal{R}^{\text{\tiny CS}}$ and a {\em disagreement} set $\mathcal{R}^{\text{\tiny DS}}$ by the predictions in the GF pathway $\PM^{\text{\tiny GF}}$:
\begin{equation*}
\mathcal{R}^{\text{\tiny CS}}=\{v_i\in \mathcal{R}|y^{\text{\tiny GA}}_i=y^{\text{\tiny GF}}_i\},\ \mathcal{R}^{\text{\tiny DS}}=\{v_i\in \mathcal{R}|y^{\text{\tiny GA}}_i\neq y^{\text{\tiny GF}}_i\}=\mathcal{R}\setminus \mathcal{R}^{\text{\tiny CS}}
\end{equation*}
where $y^{\text{\tiny GA}}_i=\argmax{1\le k\le K}\PM^{\text{\tiny GA}}_{i,k}$ and $y^{\text{\tiny GF}}_i=\argmax{1\le k\le K}\PM^{\text{\tiny GF}}_{i,k}$ represent the predicted labels for $v_i$ by both pathways.
Intuitively, in the consensus set $\mathcal{R}^{\text{\tiny CS}}$, both pathways produce the same label predictions, and thus, they can be directly used as pseudo-labels, i.e.,
\begin{equation*}
\mathcal{U} = \mathcal{U}\cup \mathcal{R}^{\text{\tiny CS}}, \mathcal{U}^\text{\tiny GA} = \mathcal{U}^\text{\tiny GA}\cup \mathcal{R}^{\text{\tiny CS}}, \text{and}\ \mathcal{U}^\text{\tiny GF} = \mathcal{U}^\text{\tiny GF}\cup \mathcal{R}^{\text{\tiny CS}}.
\end{equation*}
As for the disagreement set $\mathcal{R}^{\text{\tiny DS}}$, we further divide it into two subsets $\mathcal{R}^{\text{\tiny DS}}_{\text{\tiny GA}}$ and $\mathcal{R}^{\text{\tiny DS}}_{\text{\tiny GF}}$, where each node $v_i$ in the former set has a higher confidence score $\omega^{\text{\tiny GA}}=\max_{1\le k\le K}\PM^{\text{\tiny GA}}_{i,k}$ from the GA pathway compared to the GF one ($\omega^{\text{\tiny GA}}=\max_{1\le k\le K}\PM^{\text{\tiny GA}}_{i,k}$), while the opposite is true for the latter. The pseudo-labels of the nodes in $\mathcal{R}^{\text{\tiny DS}}_{\text{\tiny GA}}$ and $\mathcal{R}^{\text{\tiny DS}}_{\text{\tiny GF}}$ are thus determined by $\PM^{\text{\tiny GA}}$ and $\PM^{\text{\tiny GF}}$, respectively. Accordingly, the labeled node sets are expanded as follows:
\begin{equation*}
\mathcal{U}=\mathcal{U}\cup \mathcal{R}^{\text{\tiny DS}}_{\text{\tiny GA}} \cup \mathcal{R}^{\text{\tiny DS}}_{\text{\tiny GF}}, \mathcal{U}^\text{\tiny GA} = \mathcal{U}^\text{\tiny GA}\cup \mathcal{R}^{\text{\tiny DS}}_{\text{\tiny GA}},\ \text{and}\ \mathcal{U}^\text{\tiny GF} = \mathcal{U}^\text{\tiny GF}\cup \mathcal{R}^{\text{\tiny DS}}_{\text{\tiny GF}}.
\end{equation*}

\subsection{\bf WSD-based Graph Sketching}\label{sec:sketch}
After the obtainment of soft labels $\PM$, the next step in \algo{} is to construct a {\em sparse} sketching matrix $\SM \in \mathbb{R}^{N \times N'}$ such that the sketches of $\AM$, $\XM$, and label matrix of $\G$ can be used as the adjacency matrix $\AM'\in \mathbb{R}^{N'\times N'}$, attribute matrix $\XM' \in \mathbb{R}^{N'\times D}$, and label matrix $\YM'\in\mathbb{R}^{N'}$ of the condensed TAG $\G'$, respectively. More precisely,
\begin{equation}\label{eq:sketch}
\AM' = \SM^\top \AM \SM,\ \XM' = \SM^\top \XM.
\end{equation}
and for each node $v_i\in \G^\prime$, its synthetic label is constructed as follows:
\begin{equation}\label{eq:label-assign}
% \YM'_{i,k} = \arg\max_{\text{row}} (\SM^\top \PM)
\YM'_{i,k} = \begin{cases} 1 & \text{if } k = \argmax{1\le k\le K}{(\SM^\top \PM)}_{i,:} \\ 0 & \text{otherwise} \end{cases}.
\end{equation}

Akin to the count-sketch~\cite{clarkson2017low}, we impose a {\em one-hot constraint} over $\SM$ for sparsity, wherein each row in $\SM$ contains a single non-zero entry indicating the hashing or sampling of the node into a bucket.

\stitle{Sketching Objective} Inspired by our empirical observations in \S\ref{sec:observations_objectives}, a desired $\G^\prime$ obtained by Eq.~\eqref{eq:sketch} should have small WSD distances from $\G$ in terms of textual attributes and graph topology, leading to the following objectives:
\begin{equation}\label{eq:WSD-obj}
% \begin{gathered}
\min_{\SM}{\textsf{WSD}(f(\AM), f(\AM^\prime))},\ \text{and}\ \min_{\SM}{\textsf{WSD}(\XM, \XM^\prime)},
% \end{gathered}
\end{equation}
where $f(\cdot)$ stands for an embedding function that maps $\AM_i \in \mathbb{R}^{N}$ and $\AM'_i \in \mathbb{R}^{N'}$ to vectors of the same dimensions.
Since Fig.~\ref{fig:acc-wsd} shows that $\textsf{WSD}(\textsf{GNN}(\AM,\XM), \textsf{GNN}(\AM^\prime,\XM^\prime))$ exhibits the same trends as that of the single modality, and Fig.~\ref{fig:single_col} uncovers that $\textsf{GNN}(\AM,\XM)$ fails to fully capture the features in $\textsf{MLP}(\XM)$, we thus transform Eq.~\eqref{eq:WSD-obj} into the joint minimization of $\textsf{WSD}(\textsf{GNN}(\AM,\XM), \textsf{GNN}(\AM^\prime,\XM^\prime))$ and $\textsf{WSD}(\textsf{MLP}(\XM), \textsf{MLP}(\XM^\prime))$.

\begin{lemma}\label{lem:WSD-clust}
For any matrix $\MM$, the minimization of $\textsf{\textnormal{WSD}}(\MM,\SM^\top\MM)$ is equivalent to $\min_{\SM}{\frac{1}{N} \sum_{k=1}^{N'} \sum_{i=1}^{N} \mathbb{1}_{\SM_{k,i}\neq 0}\cdot\left\| \MM_i - \sum_{j=1}^{N} \SM_{k,j} \MM_j \right\|_2^2}$.
\end{lemma}

Since $\SM$ has a single non-zero entry per row, the sketching matrix $\SM$ can be represented as a set of disjoint clusters $\{\C_1,\C_2,\ldots,\C_{N^\prime}\}$ where $\mathcal{C}_k=\{v_i\in \V|\ \SM_{k,i}>0\}$.
Accordingly, $\mathbb{1}_{\SM_{k,i}\neq 0}$ can be seen as a cluster indicator and $\sum_{j=1}^{N} \SM_{k,j} \MM_j=\sum_{v_j\in \mathcal{C}_k}{\SM_{k,j} \MM_j}$ can be perceived as the weighted centroid of the cluster $\mathcal{C}_k$. The objective in Lemma~\ref{lem:WSD-clust} can be rewritten as
\begin{small}
\begin{equation*}
\min_{\{\C_1,\C_2,\ldots,\C_{N^\prime}\}}{\frac{1}{N} \sum_{k=1}^{N'} \sum_{v_i\in \mathcal{C}_k}\left\| \MM_i - \sum_{v_j\in \mathcal{C}_k} \SM_{k,j} \MM_j \right\|_2^2}.
\end{equation*}
\end{small}
In other words, Lemma~\ref{lem:WSD-clust}\footnote{The proof can be found in Appendix \ref{appendix:theory}} implies that the construction of sketching matrix $\SM$ towards optimizing our above WSD-based objective can be equivalently perceived as the clustering over feature vectors $\textsf{GCN}(\AM, \XM)=\HM^{\text{\tiny GA}}$ and $\textsf{MLP}(\XM)=\HM^{\text{\tiny GF}}$.

According to Eq.~\eqref{eq:P-def}, $\PM$ adaptively combines $\HM^{\text{\tiny GA}}$ and $\HM^{\text{\tiny GF}}$ via attention weights. As such, instead of joint clustering over both features, we simply cluster vectors in $\PM$, leading to the following optimization objective for deriving the sketching matrix $\SM$:
\begin{small}
\begin{equation}\label{eq:logit_kmeans}
\min_{\{\C_1,\C_2,\ldots,\C_{N^\prime}\}}{\frac{1}{N} \sum_{k=1}^{N'} \sum_{v_i\in \mathcal{C}_k}\left\| \PM_i - \sum_{v_j\in \mathcal{C}_k} \SM_{k,j} \PM_j \right\|_2^2}.
\end{equation}
\end{small}

\stitle{Construction of $\SM$}
The objective function in Eq.~\eqref{eq:logit_kmeans} is a {\em within-cluster sum of squares}, if we set the weight $\SM_{k,j}$ of each node $v_j$ to $\frac{1}{|\mathcal{C}_k|}$, which can be solved by the well-known {K-Means} algorithm~\cite{lloyd1982least}. 
However, the class labels of nodes are obtained by {\em hardmax discretization} of softmax distributions $\PM$ or $\SM^\top\PM$ (see Eq.~\eqref{eq:label-assign}). The clustering of softmax distributions is likely to group nodes with distinct labels into the same cluster. 
To alleviate this issue, we additionally employ a {\em greedy reassignment} scheme as a post-refinement of $\SM$. 
More concretely, we first evaluate the affinity between each class label $y_k\in \mathcal{Y}$ ($1\le k\le K$) and each cluster $\mathcal{C}_j\in \{\C_1,\C_2,\ldots,\C_{N^\prime}\}$ by computing the fraction of the total population of label $y_k$ that falls into $\mathcal{C}_j$: 
\begin{small} 
\begin{equation*} 
\frac{|\{v_i\in \mathcal{C}_j|\ \YM_{i,k}=1\}|}{|\{v_i\in \V|\ \YM_{i,k}=1\}|}. 
\end{equation*} 
\end{small} 
We then pick, for each class, the class-cluster pair with the highest affinity value, and then select the remaining class-cluster pairs with the highest affinity values among all the $K\times N^\prime$ pairs until obtaining top-$N^{\prime}$ pairs in total. The nodes in each of these classes will be finally assigned to the corresponding cluster. 
For each node $v_i$ without a final assignment in the above step, its cluster is updated to the one whose centroid has the minimum distance to $v_i$'s label distribution among the selected clusters of the same predicted class.

\subsection{Keywords-based Text Synthesis with LLMs}\label{sec:text-synthesis}

It is necessary to generate corresponding readable texts of each node in $\G'$, which can be utilized in downstream LLM-based tasks and make it easier for human understanding. Notice that the sparse sketching matrix $\SM$ obtained in the preceding section essentially partitions the nodes in $\G$ into disjoint clusters $\{\C_1,\C_2,\ldots,\C_{N^\prime}\}$, each of which corresponds to a node in $\G^\prime$. 
Intuitively, the textual attributes of each node $v^{\prime}_i\in\G^\prime$ should be a summarization of the text sequences of nodes in its corresponding cluster $\C_i$.  Traditional text distillation maps the synthetic embeddings to discrete synthetic texts~\cite{morris2023text} via {\em vector to text} (V2T), which may involve poor cross-architecture generalization, high tuning overhead, and low interpretability. 

To address these limitations, we directly invoke an LLM to summarize the original texts corresponding to nodes within each cluster, obtaining the synthetic texts without intermediate embedding-to-text mapping. However, directly processing clustered texts via LLMs is both highly time-consuming and financially costly. To cost-effectively leverage the massive knowledge and remarkable generative abilities of LLMs for the summarization fulfilling the above goals, \algo{} resorts to a three-phase approach.

\stitle{\bf Cluster-based Keyword Extraction}
Instead of simply assembling all the text sequences in each cluster $\C_i$ as the input to LLMs, which not only comprises substantial noisy and redundant information but also entails a high query cost, our solution is to first sift out the keywords from the text sequences $\mathcal{S}_i = \{s_1, s_2, \dots, s_{|\C_i|}\}$ of each cluster $\C_i$.
Towards this end, we define the {\em intra-cluster} TF-IDF for each word $w$ as follows:
\begin{equation}
\text{TF-IDF}(w, \mathcal{S}_i) = \text{TF}(w, \mathcal{S}_i) \times \text{IDF}(w, \mathcal{S}_i),
\end{equation}
where $\text{TF}(w, \mathcal{S}_i)$ and $\text{IDF}(w, \mathcal{S}_i)$ are formulated by 
\begin{small}
\begin{equation*}
\begin{gathered}
\text{TF}(w, \mathcal{S}_i) = \frac{1}{|\C_i|} \sum_{k=1}^{|\C_i|} \frac{\#(w, s_k)}{\underset{w' \in s_k}{\max} \#(w', s_k)},\ \text{IDF}(w, \mathcal{S}_i) = \log\left( \frac{|\C_i|}{1 + \text{df}(w, \mathcal{S}_i)} \right).
\end{gathered}
\end{equation*}
\end{small}
Particularly, $\#(w, s_j)$ is the frequency of word $w$ in text $s_j$, and $\text{df}(w, \mathcal{S}_i)$ denotes the number of texts in $\mathcal{S}_i$ containing word $w$.

Accordingly, we next pick the top-$2\xi$ words with the highest intra-cluster TF-IDF values in $\mathcal{S}_i$, followed by an additional filtering step to prune the $\xi$ words with the smallest Euclidean distances (SBERT embedding vectors) to the corresponding attribute vector $\XM^\prime_i$ of $\C_i$. Finally, we obtain a list of $\xi$ distinct words $\mathcal{W}_i = \{w_{i,1}, w_{i,2}, \dots, w_{i,\xi}\}$ as the keywords for cluster $\C_i$ or condensed node $v^{\prime}_i$ in $\G^\prime$. Notably, our keyword extraction jointly considers words’ statistical importance (quantified by intra-cluster TF-IDF) and semantic relevance (retained via embedding distance). Using compact keywords instead of full sentences allows LLMs to capture more entities within the same input length. Keywords as input can help with interpretable text synthesis and enable flexible output by adjusting keyword order or composition.

\stitle{\bf Candidate Text Generation with LLMs}
Based on keywords $\mathcal{W}_i$, we prompt an \text{LLM} to generate human-readable summary text $s^\prime_{i,j}$ as a candidate text for $v^{\prime}_i\in \G^\prime$:
\begin{equation}
s^\prime_{i,j} = \textsf{LLM}(\mathcal{P},\mathcal{W}_i,y^\prime_i),
\end{equation}
where $\mathcal{P}$ signifies the prompt specifying the task instruction and a formal description of $\G$'s domain, task, and attribute characteristics, and $y^\prime_i$ is the semantic label of $v^\prime_i$ derived in Eq.~\eqref{eq:sketch}.
In particular, \algo{} prompts the LLMs to generate a set of $q$ (typically $q=3$) candidate text $\{s^\prime_{i,1}, s^\prime_{i,2}, \dots, s^\prime_{i,q}\}$ for $v^\prime_i$.
% For the interest of space, we refer interested readers to Appendix~\ref{app:prompts} for the detailed prompt templates.

\stitle{WSD-Validation-guided Approximate Selection}
Since each node in $\G^\prime$ has $q$ candidate text sequences, it is still prohibitively expensive to find the best synthetic textual attributes due to the $O(q^{N^\prime})$ possible combinations.
As a workaround, \algo{} samples $n_s$ combinations of the candidate text for nodes in $\G^\prime$ uniformly at random, yielding $n_s$ candidate graphs with distinct text sequences $\{s^{(j)\prime}_{1}, s^{(j)\prime}_{2}, \dots, s^{(j)\prime}_{N^\prime}\}_{j=1}^{n_s}$. 
Based thereon, each set of text sequences is converted into an attribute matrix $\XM^{(j)\prime}$ via a pre-trained text encoder (i.e., SBERT). We select the top-$k$  (often $k=5$) candidates with the smallest WSD to $\XM$, evaluate their downstream task performance on the original dataset, and choose the one achieving the highest validation accuracy as the final textual attributes for $\G^\prime$.
% According to our empirical analysis in Appendix \ref{exp:hyperpara}, using a small value for $n_s$ (often $n_s=100$) is sufficient.

\begin{table}[t]
\centering
% \footnotesize
\small 
\caption{Statistics of the seven TAG datasets.}\label{table:dataset}
\vspace{-3ex}
\renewcommand{\arraystretch}{0.9}
% \resizebox{\linewidth}{!}{
\begin{tabular}{lcccc}
\toprule
\textbf{Name}  & \textbf{\#Nodes} & \textbf{\#Edges} & \textbf{Domain} & \textbf{\#Classes} \\
\midrule
\textit{Cora}                     & 2708             & 10556            & CS Citation           & 7                  \\
\textit{CiteSeer}                       & 3186             & 8450             & CS Citation          & 6                  \\
\textit{DBLP  }                         & 14376            & 431326           & CS Citation           & 4                  \\
\textit{Computers }                     & 87229            & 1256548          & E-commerce            & 10                 \\
\textit{Photo}                          & 48362            & 873782           & E-commerce            & 12                 \\
\textit{History}                        & 41551            & 503180           & E-commerce            & 12                 \\
\textit{WikiCS }                        & 11701            & 431726           & Knowledge             & 10                 \\
\bottomrule
\end{tabular}
% }
\vspace{-2ex}
\end{table}

\begin{table*}[htbp]
\renewcommand{\arraystretch}{0.95}
  \centering
  % \footnotesize
  \caption{Node classification accuracy over distillation ratios. Best and runner-up per cell, in green and light green, respectively. Using Qwen3-max for text synthesis.}
  \vspace{-2ex}
  \label{tab:distillation_final}
  \resizebox{\linewidth}{!}{
\addtolength{\tabcolsep}{-0.3em}
    \begin{tabular}{lc| *{12}{c} |cc}
    \toprule
    \multirow{2}{*}{\textbf{Dataset}} & \multirow{2}{*}{\textbf{$r$}} & \multicolumn{12}{c}{\textbf{Distillation Methods}} & \multicolumn{2}{|c}{\textbf{Full Methods}}\\
    % \cmidrule(lr){3-14} \cmidrule(lr){16-17} 
    \cline{3-14} \cline{15-16}
    & & \makecell{\text{\GCond}} & \makecell{\text{\GCDM}} & \makecell{\text{\GDEM}} & \makecell{\text{\ClustGDD}} & \makecell{\text{\CGC}} & \makecell{\text{\TCond}} & \makecell{\text{\TCDM}} & \makecell{\text{\ASD}} & \makecell{\text{\ClustTDD}} & \makecell{\text{\DaLLME}} & \makecell{\text{\algo{}($\XM'$)}} &\makecell{\text{\algo{}($\mathcal{S}'$)}} & \text{GCN} & \text{MLP} \\
    \midrule
    \multirow{3}[2]{*}{\rotatebox[origin=c]{60}{\em Cora}} 
    & 1.0    & 72.8$\pm$2.0 & 76.2$\pm$1.6  & \cellcolor{teal!30}78.7$\pm$0.8 & \cellcolor{teal!10}77.9$\pm$1.4 & 74.9$\pm$0.5 & 66.1$\pm$2.6 & 72.9$\pm$2.1  & 73.8$\pm$1.5 & 75.1$\pm$1.4 & 66.8$\pm$2.0 & \cellcolor{teal!50}{79.5$\pm$1.8} & 76.9$\pm$0.5 &  \multirow{3}[2]{*}{77.7$\pm$2.0} & \multirow{3}[2]{*}{72.8$\pm$ 1.0} \\
    & 0.5    & 71.5$\pm$2.3 & 73.7$\pm$2.0 & 72.9$\pm$2.0 & \cellcolor{teal!30}77.7$\pm$1.7 & 74.1$\pm$0.8 & 63.1$\pm$2.7 & 61.3$\pm$1.8 & 71.7$\pm$ 1.9 & 73.9$\pm$1.4 & 64.8$\pm$2.3 & \cellcolor{teal!50}{79.1$\pm$1.1} & \cellcolor{teal!10}{76.1$\pm$0.8} &  & \\
    & 0.25   & 63.5$\pm$5.6  & 68.9$\pm$2.8 & 76.3$\pm$1.1 & \cellcolor{teal!30}76.6$\pm$1.3 & 72.6$\pm$1.1& 63.1$\pm$1.8 & 66.0$\pm$1.5 & 66.9$\pm$1.2 & 67.8$\pm$3.5 & 60.0$\pm$1.1 & \cellcolor{teal!50}{79.1$\pm$0.9} & \cellcolor{teal!10}{76.5$\pm$0.5} &  \\
     & 0.05   & 45.8$\pm$5.2 & 47.7$\pm$6.4 & 73.1$\pm$1.6 & \cellcolor{teal!10}74.2$\pm$2.4  &\cellcolor{teal!30} 74.9$\pm$0.9 & 55.1$\pm$2.3 & 48.0$\pm$5.3 & 46.9$\pm$7.0 & 69.2$\pm$4.4 & 52.1$\pm$1.6 & \cellcolor{teal!50}77.5$\pm$0.6 & 73.9$\pm$0.7 &  &  \\
    \midrule
    \multirow{3}[2]{*}{\rotatebox[origin=c]{60}{\em Citeseer}} 
    & 1.0    & 75.1$\pm$0.6 & 75.4$\pm$1.3 & 75.7$\pm$0.8  & \cellcolor{teal!10}76.1$\pm$0.6 & 73.7$\pm$0.9 & 66.0$\pm$3.3 & 72.5$\pm$0.9  & 72.6$\pm$1.0 & 72.4$\pm$0.8 & 70.3$\pm$3.4 & \cellcolor{teal!50}76.3$\pm$0.4 & \cellcolor{teal!50}76.3$\pm$0.6  &  \multirow{3}[2]{*}{74.0$\pm$2.1} & \multirow{3}[2]{*}{71.9$\pm$1.5} \\
    & 0.5    & 69.2$\pm$2.2 & 74.8$\pm$0.9 & 73.5$\pm$1.7 & \cellcolor{teal!10}73.9$\pm$1.5 & 73.5$\pm$1.5 & 64.3$\pm$0.5 & 70.4$\pm$1.5 & 70.6$\pm$1.3 & 72.7$\pm$0.8 & 67.5$\pm$2.3 &  \cellcolor{teal!50}76.3$\pm$1.1 & \cellcolor{teal!30}75.4$\pm$0.6 & &  \\
    & 0.25   & 66.6$\pm$9.0   & 70.2$\pm$1.1 & \cellcolor{teal!30}75.6$\pm$0.6 & 72.9$\pm$2.1 & 73.1$\pm$0.9 &58.4$\pm$1.3 & 64.7$\pm$3.0 & 67.5$\pm$3.7 & 68.2$\pm$0.7 &  63.1$\pm$3.5 & \cellcolor{teal!50}76.0$\pm$1.6 & \cellcolor{teal!10}75.3$\pm$1.5 &  & \\
    & 0.05   & 48.3$\pm$7.6  & 51.5$\pm$9.0 & \cellcolor{teal!50}76.8$\pm$0.5 & 71.7$\pm$3.0 & 75.1$\pm$0.5 & 58.3$\pm$2.7 & 55.7$\pm$6.2 & 50.6$\pm$3.2 & 66.9$\pm$0.5 & 63.0$\pm$3.5 & \cellcolor{teal!30}75.9$\pm$0.5 & \cellcolor{teal!10}73.2$\pm$0.9  & &  \\
    \midrule
    \multirow{3}[2]{*}{\rotatebox[origin=c]{60}{\em DBLP}} 
    & 1.0    & 63.0$\pm$5.8  & 77.4$\pm$0.6 & 76.7$\pm$0.3 & \cellcolor{teal!30}78.4$\pm$0.4 & 76.2$\pm$0.8 & 61.9$\pm$1.2 & 68.8$\pm$1.1 & 68.7$\pm$0.5 &  69.4$\pm$1.1 & 67.6$\pm$7.3 & \cellcolor{teal!10}78.3$\pm$0.3  & \cellcolor{teal!50}78.5$\pm$0.4 & \multirow{3}[2]{*}{77.9$\pm$0.5} & \multirow{3}[2]{*}{68.0$\pm$0.8} \\
    & 0.5    & 67.5$\pm$4.2  & 75.3$\pm$1.0 & 74.3$\pm$2.0  & \cellcolor{teal!10}77.8$\pm$0.4 & 63.1$\pm$3.6 & 58.1$\pm$2.2 & 65.7$\pm$1.4 & 64.9$\pm$1.5 &  67.9$\pm$2.1 & 67.1$\pm$2.2 & \cellcolor{teal!50}78.4$\pm$0.3  & \cellcolor{teal!30}78.1$\pm$0.3  & &  \\
    & 0.25   & 61.9$\pm$12.1 & 71.6$\pm$3.7 & \cellcolor{teal!10}76.3$\pm$0.5 & 75.3$\pm$1.5 & 52.5$\pm$10.1 & 56.8$\pm$2.7 &60.5$\pm$2.2  & 61.2$\pm$2.3 & 67.0$\pm$2.2 & 65.3$\pm$1.3 & \cellcolor{teal!50}78.3$\pm$0.4 & \cellcolor{teal!30}77.8$\pm$0.5  &  & \\
    & 0.05   & 70.3$\pm$4.0 & 55.7$\pm$5.7 & \cellcolor{teal!10}75.9$\pm$0.7 & 74.6$\pm$2.8 & 74.4$\pm$4.1 & 52.1$\pm$2.8 &46.0$\pm$4.7  & 43.0$\pm$3.6 & 67.0$\pm$2.4 & 54.4$\pm$0.6 & \cellcolor{teal!50}77.2$\pm$0.2 & \cellcolor{teal!30}76.4$\pm$0.8 &  & \\
    \midrule
    \multirow{3}[2]{*}{\rotatebox[origin=c]{60}{\em Computers}} 
    & 1.0    & 55.8$\pm$10.9  & 66.7$\pm$0.7 & 61.7$\pm$1.5  & \cellcolor{teal!10}68.2$\pm$3.1 & 63.0$\pm$1.0 & 38.6$\pm$1.5 & 49.8$\pm$2.0  & 49.4$\pm$2.2 & 52.0$\pm$1.7 & 48.3$\pm$1.5 & \cellcolor{teal!50}73.5$\pm$1.5 & \cellcolor{teal!30}70.2$\pm$1.3 & \multirow{3}[2]{*}{72.2$\pm$1.9} & \multirow{3}[2]{*}{48.9$\pm$1.3} \\
    & 0.5    & 57.4$\pm$3.7 & 62.6$\pm$3.2 & 62.7$\pm$0.9 & \cellcolor{teal!10}68.1$\pm$2.8 & 55.1$\pm$1.4 & 37.1$\pm$1.4 & 43.4$\pm$2.8  & 44.2$\pm$3.0 & 50.5$\pm$3.3 & 45.5$\pm$2.7 & \cellcolor{teal!50}74.2$\pm$1.0 & \cellcolor{teal!30}70.3$\pm$1.1  &  & \\
    & 0.25   & 48.0$\pm$6.3 & 53.9$\pm$5.1 & 54.2$\pm$4.1 & \cellcolor{teal!10}61.9$\pm$3.5 & 50.5$\pm$3.5 & 37.9$\pm$1.7 & 39.2$\pm$2.1  & 39.8$\pm$1.0 & 42.9$\pm$3.0 & 41.4$\pm$3.2 & \cellcolor{teal!50}71.5$\pm$2.8 & \cellcolor{teal!30}68.2$\pm$0.8  &  & \\
    & 0.05   & 39.9$\pm$7.7  & 35.8$\pm$3.9 & 60.7$\pm$4.1 & \cellcolor{teal!10}61.3$\pm$5.8  & 57.9$\pm$0.7 & 36.4$\pm$1.8 &30.0$\pm$2.5   & 30.0$\pm$1.8 &  43.3$\pm$6.0 & 27.7$\pm$2.2 & \cellcolor{teal!50}65.4$\pm$3.7 & \cellcolor{teal!30}62.6$\pm$2.5  & &  \\
    \midrule
    \multirow{3}[2]{*}{\rotatebox[origin=c]{60}{\em Photo}} 
    & 1.0    & 61.7$\pm$2.0  & 62.2$\pm$2.6 & 61.4$\pm$1.6 & \cellcolor{teal!30}69.5$\pm$1.8 & 65.7$\pm$1.2 & 44.3$\pm$2.2 & 52.5$\pm$2.6 & 53.5$\pm$2.4 & 52.0$\pm$5.7 & 50.5$\pm$3.4 & \cellcolor{teal!50}71.2$\pm$0.9 & \cellcolor{teal!10}68.7$\pm$1.2 & \multirow{3}[2]{*}{70.0$\pm$0.9} & \multirow{3}[2]{*}{51.3$\pm$2.5} \\
    & 0.5    & 54.8$\pm$7.4 & 59.5$\pm$1.4 & 61.5$\pm$2.6 & \cellcolor{teal!30}70.1$\pm$3.0 & 50.2$\pm$5.1 & 44.5$\pm$2.5 & 49.8$\pm$3.7 & 51.1$\pm$1.4 & 49.7$\pm$8.3 & 46.5$\pm$3.1 & \cellcolor{teal!50}70.3$\pm$1.1 & \cellcolor{teal!10}68.3$\pm$0.8 &  & \\
    & 0.25   & 49.3$\pm$5.8 & 50.0$\pm$3.1 & 54.6$\pm$1.0 & \cellcolor{teal!10}68.8$\pm$2.2 & 46.4$\pm$2.7 & 45.6$\pm$2.7 & 46.2$\pm$5.1  & 48.3$\pm$3.4  & 42.0$\pm$7.8 & 47.0$\pm$1.7 & \cellcolor{teal!50}69.9$\pm$1.8 & \cellcolor{teal!30}68.9$\pm$1.3 &  & \\
    & 0.05   & 44.0$\pm$4.4 & 40.6$\pm$2.3 & 59.3$\pm$3.0 & 59.1$\pm$3.7  & \cellcolor{teal!50}65.7$\pm$0.6 & 45.3$\pm$1.7 & 41.7$\pm$0.4 & 41.4$\pm$0.9 & 46.8$\pm$5.6 & 40.5$\pm$3.4 & \cellcolor{teal!10}63.4$\pm$1.7 & \cellcolor{teal!30}64.6$\pm$1.4 &  & \\
    \midrule
    \multirow{3}[2]{*}{\rotatebox[origin=c]{60}{\em History}} 
    & 1.0    & 76.5$\pm$1.0 & 74.2$\pm$1.7 & 70.9$\pm$2.1 & 74.2$\pm$3.0 & \cellcolor{teal!10}74.5$\pm$0.9 & 72.1$\pm$1.0 & 69.4$\pm$2.1  & 72.9$\pm$0.7 & 74.6$\pm$1.3 & 74.2$\pm$2.2 & \cellcolor{teal!30}76.8$\pm$1.6 & \cellcolor{teal!50}78.1$\pm$1.1  & \multirow{3}[2]{*}{70.6$\pm$3.6} & \multirow{3}[2]{*}{62.8$\pm$3.7} \\
    & 0.5    & 74.6$\pm$2.8 & 71.4$\pm$1.8 & 73.5$\pm$1.2 & 74.2$\pm$3.1 & 71.9$\pm$1.6 & 69.4$\pm$2.8 & 65.6$\pm$2.8 & 68.2$\pm$3.7 & 72.1$\pm$1.5 & \cellcolor{teal!10}75.4$\pm$2.7 & \cellcolor{teal!50}78.6$\pm$0.8 & \cellcolor{teal!30}77.8$\pm$0.8 &  & \\
    & 0.25   & 72.2$\pm$1.4  & 68.2$\pm$4.8 & 66.6$\pm$2.5 & 72.6$\pm$1.8 & 70.4$\pm$1.6 & 67.3$\pm$1.0 & 62.0$\pm$5.0 & 63.4$\pm$1.8 & 70.4$\pm$2.6 & \cellcolor{teal!10}74.2$\pm$2.3 & \cellcolor{teal!30}77.4$\pm$2.8 & \cellcolor{teal!50}78.3$\pm$1.3  &  & \\
    & 0.05   & 65.9$\pm$6.0  & 58.9$\pm$3.5 & 71.9$\pm$1.4 & \cellcolor{teal!30}74.8$\pm$2.8 & 54.6$\pm$3.5 & 52.0$\pm$4.5 & 58.0$\pm$3.1 & 58.0$\pm$2.3 & 70.2$\pm$3.8 & 64.8$\pm$5.8 & \cellcolor{teal!50}75.9$\pm$1.7 & \cellcolor{teal!10}73.4$\pm$2.8  &  & \\
    \midrule
    \multirow{3}[2]{*}{\rotatebox[origin=c]{60}{\em WikiCS}} 
    & 1.0    & 50.7$\pm$16.2  & \cellcolor{teal!30}75.4$\pm$1.6 & 71.7$\pm$0.6 & 70.8$\pm$2.1 & 69.7$\pm$1.5 & 55.0$\pm$1.3 & 66.8$\pm$1.1 & 66.2$\pm$1.3 & 66.3$\pm$1.6 & 62.5$\pm$2.7 & \cellcolor{teal!50}77.6$\pm$0.4 & \cellcolor{teal!10}74.9$\pm$0.9 & \multirow{3}[2]{*}{75.9$\pm$1.2} & \multirow{3}[2]{*}{66.3$\pm$1.5} \\
    & 0.5    & 58.4$\pm$4.5  &  71.1$\pm$2.5  & \cellcolor{teal!10}71.4$\pm$1.1 & 66.6$\pm$3.0 & 57.1$\pm$1.6 & 52.6$\pm$2.7 &  61.3$\pm$2.8 & 61.6$\pm$1.3 & 63.5$\pm$1.6 & 59.8$\pm$2.7 & \cellcolor{teal!50}76.8$\pm$0.7 & \cellcolor{teal!30}73.7$\pm$0.9 & \\
    & 0.25    & 43.9$\pm$4.6  & 68.6$\pm$3.3 & \cellcolor{teal!10}68.8$\pm$1.0 & 67.2$\pm$1.5 & 47.3$\pm$3.3 & 50.6$\pm$1.7 &  55.6$\pm$2.3  & 53.9$\pm$1.3 & 57.9$\pm$1.8 & 54.0$\pm$3.3 & \cellcolor{teal!50}76.5$\pm$0.9 & \cellcolor{teal!30}74.0$\pm$0.8 &  \\
    & 0.05    & 49.3$\pm$3.1  & 40.2$\pm$3.4  & 52.0$\pm$3.4 & 63.7$\pm$3.9 & \cellcolor{teal!10}71.4$\pm$2.1 & 47.0$\pm$1.2 & 39.0$\pm$4.1 & 38.0$\pm$4.4 &  60.2$\pm$2.2  & 48.5$\pm$3.4 & \cellcolor{teal!50}76.6$\pm$0.3 & \cellcolor{teal!30}71.9$\pm$2.5 & \\
    \bottomrule
    \end{tabular}
}
\end{table*}

\section{Experiments}
We conduct experiments to address the following questions. First, how does the \algo{} perform on semi-supervised TAGD tasks? Second, how can $\G'$ generated by \algo{} help LL4TAG methods? Third, how do the modules in \algo{} contribute to the overall performance? 

% Our code is available at \url{https://github.com/laiyurui/STAD-Semi-Supervised-Text-Attributed-Graph-Distillation.git}. 

\subsection{Experimental Settings}
\stitle{Dataset}
Table \ref{table:dataset} summarizes TAG datasets: citation networks {\em Cora}, {\em Citeseer}~\cite{yang2016revisiting}, {\em DBLP}~\cite{ji2010graph}; E-commerce co-view networks {\em Computers}, {\em Photo}, {\em History}~\cite{yan2023comprehensive}; and {\em WikiCS}~\cite{yan2023comprehensive} (Wikipedia CS pages). 

\stitle{Baselines}
We compare with the following recent baselines: 
\begin{itemize}[leftmargin=*]
\item Graph distillation: \GCond~\cite{jingraph}, \GCDM~\cite{liu2022graph}, \GDEM~\cite{liu2024graph}, \ClustGDD~\cite{lai2025simple}, \CGC~\cite{gao2025rethinking};
\item Text distillation: \ASD~\cite{maekawa-etal-2023-dataset}, \DaLLME~\cite{tao2024textual}. We also include topology-removed variants \TCond, \TCDM, \ClustTDD (replacing $\AM,\AM'$ with identity matrices $\IM,\IM'$, respectively).
\end{itemize}
 
\algo{} has two variants: \algo{}($\XM'$) and \algo{}($\mathcal{S}'$). \algo{}($\XM'$) trains a GCN on $\XM'$, $\AM'$, $\YM'$ and evaluates it on $\XM$, $\AM$, $\YM$. For \algo{}($\mathcal{S}'$), $\mathcal{S}'$ is remapped to $\XM''\in \mathbb{R}^{N'\times D}$ via SBERT, then a GCN is trained on $\XM''$, $\AM'$, $\YM'$ and evaluated on $\XM$, $\AM$, $\YM$.

\stitle{Implementation Details}
We use five random splits with 20 nodes per class (or all if fewer exist). GCN and MLP are two-layer (hidden 64, lr 0.01, weight decay $0.0005$). Text synthesis uses Qwen3-max (API); LLM4TAG training/inference uses Qwen3-1.7B~\cite{yang2025qwen3}.
% We put our hyper-parameter analysis in Appendix~\ref{exp:hyperpara}, graph visualization in Appendix~\ref{exp:vis}, memory compression ratio in Appendix~\ref{appendix:stat}, prompts and output samples in Appendix~\ref{app:prompts}. 

\subsection{Node Classification Performance Evaluation}

Table~\ref{tab:distillation_final} validates our design: WSD-guided dual-path encoding and keyword-based text synthesis yield strong, stable compression. Graph distillation generally outperforms text-only distillation, consistent with full models (GCN vs.\ MLP). With $r=0.05$, \GCond and \GCDM fall below \TCond and \TCDM, likely due to over-smoothing on low-quality synthetic topology. Gradient-based methods (\GCond, \TCond) are unstable as $r$ decreases, e.g., \GCond on {\em Cora} drops from 72.8\%$\pm$2.0 to 45.8\%$\pm$5.2. Clustering-based methods (\ClustGDD, \CGC, \ClustTDD, \DaLLME) are more competitive and stable, and \GDEM achieves good performance via spectral alignment. \algo{}($\XM'$) and \algo{}($\mathcal{S}'$) achieve overall best accuracy. On {\em WikiCS} at $r{=}0.5$, \algo{}($\XM'$) reaches 76.8\%, outperforming full GCN (75.9\%) and the next-best distillation (71.4\%). \algo{} on distilled data sometimes exceeds full-data training, which we attribute to effective extraction of discriminative information and noise reduction. 
\algo{}($\XM'$) usually beats \algo{}($\mathcal{S}'$) because $\XM'$ preserves semantics via direct aggregation; $\XM''$ from $\mathcal{S}'$ incurs information loss from SBERT's 256-token limit. \algo{}($\mathcal{S}'$) achieves relative better performance on TAGs with large scale. This observation arises because clusters with richer text provide more context and statistical evidence, yielding more representative, informative keywords. 
% For distillation time comparison, please refer to Appendix~\ref{appendix:real-time}.

% syn node classification with LLM4TAG Methods 
\begin{table}[htbp]
\renewcommand{\arraystretch}{0.9}  
\centering
% \footnotesize
\small
\caption{Node classification accuracy (mean~\% $\pm$ std) with semi-supervised LLM4TAG methods with $r=0.5$.}
\vspace{-2ex}
\label{tab:distillation_llm4g}
\addtolength{\tabcolsep}{-0.3em}  
% \resizebox{\linewidth}{!}{
\begin{tabular}{l c *{3}{c}} 
\toprule
\multirow{2}{*}{\textbf{Dataset}} & \multirow{2}{*}{\textbf{\makecell{Distillation\\ Method}}} & \multicolumn{3}{c}{\textbf{Baselines}} \\
\cmidrule(lr){3-5}  
& & \makecell{\GNNLLM} & \makecell{\text{\ENGINE}} & \makecell{\text{\TAPE}} \\
\midrule
\multirow{7}{*}{\em Cora}       
& \text{Full Methods}          & 77.8$\pm$1.6 & 78.8$\pm$1.1 & 77.5$\pm$1.0 \\
& \text{\GCond+V2T}         & 29.4$\pm$2.4 & 35.0$\pm$7.0 & 33.5$\pm$6.0 \\
& \text{\GCDM+V2T}          & 45.1$\pm$7.9 & 47.1$\pm$6.3 & 58.7$\pm$6.8 \\
& \text{\GDEM+V2T}          & 27.4$\pm$5.9 & 25.9$\pm$6.0 & 10.4$\pm$3.3 \\
& \text{\ClustGDD+V2T}      &  \cellcolor{teal!30} 55.9$\pm$5.6 &  \cellcolor{teal!10} 41.8$\pm$6.3 &  \cellcolor{teal!30} 65.0$\pm$2.0 \\
& \text{\DaLLME}            &  \cellcolor{teal!10}53.0$\pm$5.1 &  \cellcolor{teal!30} 57.9$\pm$5.0 &  \cellcolor{teal!10}61.5$\pm$2.1 \\
& \text{\algo{}}               &   \cellcolor{teal!50} 70.9$\pm$4.2 &   \cellcolor{teal!50} 73.9$\pm$1.6 &  \cellcolor{teal!50} 70.4$\pm$2.6 \\

\midrule
\multirow{7}{*}{\em History}    
& \text{Full Methods}          & 65.4$\pm$4.8 & 69.1$\pm$2.4 & 69.0$\pm$4.9 \\
& \text{\GCond+V2T}         & 40.6$\pm$21.4& 48.0$\pm$17.9& 22.1$\pm$3.7 \\
& \text{\GCDM+V2T}          & 48.4$\pm$15.9& 46.4$\pm$19.7& 44.6$\pm$15.3 \\
& \text{\GDEM+V2T}          & 46.1$\pm$15.6& 40.6$\pm$19.5& 14.6$\pm$12.1 \\
& \text{\ClustGDD+V2T}      &  \cellcolor{teal!30}70.9$\pm$3.1 &  \cellcolor{teal!10}63.0$\pm$3.7 &  \cellcolor{teal!10}68.3$\pm$4.1 \\
& \text{\DaLLME}            &  \cellcolor{teal!10} 64.6$\pm$6.2 &  \cellcolor{teal!30}73.8$\pm$2.1 &  \cellcolor{teal!50}72.7$\pm$1.1 \\
& \text{\algo{}}               &  \cellcolor{teal!50}76.6$\pm$1.8 &  \cellcolor{teal!50}74.6$\pm$4.3 &  \cellcolor{teal!50}74.3$\pm$2.2 \\
\bottomrule
\end{tabular}
% }
\end{table}

% Node classification with training free-LLMs 
\begin{table}[htbp]
\renewcommand{\arraystretch}{0.9} 
\centering
\small
% \footnotesize
\caption{Node classification accuracy (mean~\% $\pm$ std.) with training-free LLM4TAG methods (ICL).}
\vspace{-2ex}
\label{tab:distillation_icl_tf}
\addtolength{\tabcolsep}{-0.3em}  
\begin{tabular}{l  c c| c c}  % Add vertical separator after Dataset column
\toprule
\multirow{2}{*}{\textbf{Dataset}} &  \multicolumn{4}{c}{\textbf{Methods}} \\  
\cmidrule(lr){2-5}  
 & \makecell{\text{Direct}} & \makecell{\text{\algo{}+Direct}} & \makecell{\text{NS.}} & \makecell{\text{\algo{}+NS.}} \\  
\midrule
{\em Cora}      & 59.9$\pm$0.2 & 62.3$\pm$0.6 & 60.5$\pm$0.7 & 64.1$\pm$0.7 \\ 
\midrule
 {\em History}   & 13.6$\pm$0.0 & 23.5$\pm$2.3 & 14.8$\pm$0.0 & 32.5$\pm$4.6 \\ 
\bottomrule
\end{tabular}
\vspace{-3ex}
\end{table}

\subsection{LLM4TAG Performance Evaluation}

We train three LLM4TAG methods, \GNNLLM \cite{wullms}, \ENGINE \cite{zhu2024efficient}, and \TAPE \cite{he2023harnessing}, with $r=0.5$. To ensure fair comparison, we use V2T \cite{morris2023text} to map synthetic attribute matrices to text. Table~\ref{tab:distillation_llm4g} shows that \algo{} achieves the best transfer performance to these frameworks. On {\em Cora}, \algo{} reaches 70.9\%, 73.9\%, and 70.4\% on the three methods, respectively, outperforming all baselines and narrowing the gap to the full-data baseline. On {\em History}, \algo{} attains 76.6\%, 74.6\%, and 74.3\%, surpassing the full-data baseline for \ENGINE and \TAPE. Human-readable synthetic texts preserve the semantics required by LLM4TAG models; in contrast, \GCond and \GDEM with V2T perform poorly, likely due to embedding misalignment caused by V2T.
% For preprocessing and training time comparison between full methods and methods on $\G'$, please refer to Appendix~\ref{appendix:real-time}. 

We also use \algo{}($\mathcal{S}'$) for training-free LLM4TAG, which includes two modes: Direct (i.e., LLMs directly give predictions based on node texts) and Neighbor Summarization (NS, i.e., LLMs first summarize the text of a node’s neighbors, then make predictions based on the node’s own text and the neighbor summary). We apply \algo{}($\mathcal{S}'$) via {\em in-context learning} (ICL): prompts are constructed based on the distances between input embeddings and synthetic node attributes. Table~\ref{tab:distillation_icl_tf} shows consistent performance gains: on {\em Cora}, we observe +2.4\% (Direct) and +3.6\% (Nbr.Sum.); on {\em History}, the gains are +9.9\% and +17.7\%. These results show that synthetic texts are effective ICL exemplars. 

\subsection{Ablation Study}

We conduct ablation studies with $r=0.05$ to validate the effectiveness of modules in our framework in Table~\ref{tab:ablation_datasets}. (w/o: module removing, w: another module as alternative.)
Removing the GA pathway leads to significant performance degradation, particularly on \textit{Photo}, where accuracy drops by 11.5\%, indicating the critical role of graph structure. Disabling CoST consistently reduces performance across all datasets, confirming that it effectively enhances the dual-pathway architecture.
Removing our reassignment after K-Means hurts results, validating the necessity of reassignment for more pure clusters. 
We omit graph-text collaborative encoding and employ raw attributes from SBERT for sketching. The performance decline verifies the necessity of the encoding module.  
Furthermore, applying vanilla self-training to dual-path encoders and baseline methods (e.g., GCDM+ST) yields inferior performance compared to \algo{}($\XM'$), demonstrating the superiority of Graph-Text Collaborative Encoding in semi-supervised distillation.
We evaluate alternative text generation strategies. Variants without LLM summarization, with random words, or with key sentences all underperform our keywords-based approach, with random words causing severe degradation up to 15.5\% on Photo. 
We build summarization prompts without synthetic label injection, causing performance drops and instability across most datasets. This proves that synthetic label constraints enable LLMs to produce more category-aware summaries. 
Direct generation without candidate set construction also reduces accuracy, validating the importance of candidate selection. Finally, we select candidates only via WSD. Results show both WSD scores and validation accuracy are essential for optimal candidate selection.

\begin{table}[t]
\centering
% \footnotesize
\small
\caption{Ablation study: accuracy (mean~\% $\pm$ std).}
\label{tab:ablation_datasets}
\vspace{-2ex}
\begin{tabular}{lcccc}
\toprule
Modules & \textit{Cora} & \textit{Photo} & \textit{History} & \textit{WikiCS} \\
\midrule
w/o GA                  & 75.3$\pm$1.0          & 52.0$\pm$5.6          & 70.8$\pm$3.9          & 65.9$\pm$1.5 \\
w/o GF                  & 76.6$\pm$1.6          & 61.4$\pm$1.8          & 71.8$\pm$2.6          & 75.8$\pm$0.9 \\
w/o CoST                & 75.7$\pm$0.9          & 60.6$\pm$0.6          & 64.8$\pm$4.1          & \cellcolor{teal!30}76.4$\pm$0.7 \\
w/o Reassignment        & 75.2$\pm$2.2          & 61.8$\pm$1.4          & \cellcolor{teal!30}75.4$\pm$4.0 & 68.1$\pm$3.3 \\
w SBERT Clustering        & 73.8$\pm$0.8          & 59.6$\pm$1.8          & 68.9$\pm$5.7   & 71.8$\pm$0.7 \\
w vanilla ST              & \cellcolor{teal!30}77.4$\pm$1.0 & \cellcolor{teal!10} 63.4$\pm$0.7          & 73.2$\pm$2.0   & \cellcolor{teal!10} 76.2$\pm$0.5 \\
\GCDM+ST                & \cellcolor{teal!10}77.3$\pm$2.2 & \cellcolor{teal!50}70.8$\pm$1.9 & 70.1$\pm$3.2          & 74.9$\pm$3.2 \\
\GDEM+ST                & 73.1$\pm$2.2          & 54.7$\pm$1.8          & 72.0$\pm$4.2          & 50.8$\pm$1.2 \\
\ClustGDD+ST            & 74.3$\pm$1.2          & 62.8$\pm$4.5          & \cellcolor{teal!10}74.5$\pm$3.1 & 62.8$\pm$4.9 \\
\algo{}($\XM'$)         & \cellcolor{teal!50}77.5$\pm$1.3 & \cellcolor{teal!30}63.5$\pm$1.7 & \cellcolor{teal!50}75.9$\pm$1.7 & \cellcolor{teal!50}76.6$\pm$0.3 \\
\midrule
w/o LLM Summary         & 71.8$\pm$1.3          & 62.0$\pm$1.0 & \cellcolor{teal!30}73.3$\pm$2.0 & 70.4$\pm$1.5 \\
w/o Syn Label           & \cellcolor{teal!30}73.5$\pm$1.0          & 59.9$\pm$1.9 & 67.2$\pm$6.5 & \cellcolor{teal!50}72.7$\pm$1.6 \\
w Random Words          & 69.7$\pm$2.1          & 49.1$\pm$3.3          & 63.0$\pm$7.7          & 67.3$\pm$1.9 \\
w Key Sentence          & 72.6$\pm$1.0 & 58.9$\pm$2.4          & 70.6$\pm$2.9          & 68.2$\pm$3.7 \\
w/o Candidates set      & 72.5$\pm$1.5 & \cellcolor{teal!10}62.1$\pm$5.5 & 71.6$\pm$1.0 & \cellcolor{teal!10}71.5$\pm$0.9 \\
w WSD-only       & \cellcolor{teal!10} 73.2$\pm$1.2          & \cellcolor{teal!30}63.3$\pm$1.2 & \cellcolor{teal!10}72.2$\pm$1.7 & 70.7$\pm$2.6 \\
\algo{}($\mathcal{S}'$) & \cellcolor{teal!50}73.9$\pm$0.7 & \cellcolor{teal!50}64.6$\pm$1.4 & \cellcolor{teal!50}73.4$\pm$2.8 & \cellcolor{teal!30}71.9$\pm$2.5 \\
\bottomrule
\end{tabular}
\vspace{-2ex}
\end{table}

\subsection{Evaluation of LLMs for Text Synthesis}
Table~\ref{tab:results} compares LLM performance and cost for text synthesis ~\footnote{https://www.aliyun.com/, https://www.deepseek.com/, https://openai.com} with $r=0.05$. Qwen3-max achieves the best results on both datasets (73.9\% on Cora, 73.4\% on History) at moderate cost, while smaller Qwen3 variants (1.7B--8B) offer competitive accuracy with significantly lower expense. Scaling model size within the Qwen3 series does not consistently improve performance, with 8B outperforming 32B on Cora. GPT-5 achieve comparable accuracy to top performers but at 5--7$\times$ higher price. All methods yield similar low token usage, compared to the full graphs\footnote{Token number estimated by https://pypi.org/project/tiktoken/}. Overall, cost-efficient models match or exceed expensive alternatives, suggesting model selection matters more than scale for this task.

%     % Qwen3-4B    &  (0.04, 0.17)  & 71.1$\pm$0.6      & 72.9$\pm$3.6     \\   \\
%     % Qwen3-14B   &  (0.14, 0.56)  &  70.0$\pm$1.3    &  69.5$\pm$3.0    \\
%     Gemini-2.5-Flash & 0.30 / 2.5 &  71.2$\pm$1.1     &  71.5$\pm$3.4       \\
%     Gemini-2.5-Pro  & 1.25 / 10.00   &  72.0$\pm$1.2       & 71.2$\pm$1.1     \\

\begin{table}[t]
  \centering
  \footnotesize
  \caption{Accuracy and I/O price of various LLMs.}
  \label{tab:results}
  \vspace{-2ex}
  \setlength{\tabcolsep}{3pt}
  \begin{tabular}{@{}lccccc@{}}
    \toprule
    \multirow{2}{*}{\textbf{Model}} & \textbf{I/O} & \multicolumn{2}{c}{\textit{Cora}} & \multicolumn{2}{c@{}}{\textit{History}} \\
    \cmidrule(lr){3-4} \cmidrule(l){5-6}
    & \textbf{(\$/1M)} & Acc & \#Tokens (M) & Acc & \#Tokens (M) \\
    \midrule
    Qwen3-1.7B  & 0.04/0.17 & 66.2$\pm$2.6 & 0.0226  & 72.6$\pm$1.7 & 0.0379 \\
    Qwen3-8B    & 0.07/0.28 & 71.0$\pm$0.7 & 0.0228 & 72.4$\pm$1.9 & 0.0377 \\
    Qwen3-32B   & 0.28/1.12 & 70.5$\pm$0.8 & 0.0231  & 71.8$\pm$2.0 & 0.0366 \\
    Qwen3-max   & 0.35/1.40 & \cellcolor{teal!50} 73.9$\pm$0.7 & 0.0237  & \cellcolor{teal!30} 73.4$\pm$2.8 & 0.0383 \\
    DeepSeek-V3.2 & 0.28/0.42 & 71.3$\pm$1.3 & 0.0223 & \cellcolor{teal!10} 73.1$\pm$2.8 & 0.0372 \\
    GPT-4.1-mini  & 0.40/1.60   & \cellcolor{teal!30} 72.2$\pm$1.1 & 0.0228  & 69.1$\pm$1.7 & 0.0371   \\
    GPT-4.1  & 2.00/8.00   & 72.0$\pm$1.2  &  0.0230 & 72.8$\pm$2.9  & 0.0369  \\
    GPT-5-mini  & 0.25/2.0  & 70.0$\pm$1.5 & 0.0221  & 72.2$\pm$3.9 & 0.0351 \\
    GPT-5       & 1.25/10.0 & \cellcolor{teal!30} 72.2$\pm$1.7 & 0.0232 & \cellcolor{teal!50} 73.9$\pm$1.7 & 0.0376 \\
    % Gem-2.5-Flash & 0.30/2.5 & 71.2$\pm$1.1 &  & 71.5$\pm$3.4 &  \\
    % Gem-2.5-Pro & 1.25/10.0 & 72.0$\pm$1.2 &  & 71.2$\pm$1.1 &  \\
    \midrule
    \multicolumn{2}{@{}l}{\text{\#Tokens for entire TAGs (M)}} & \multicolumn{2}{c}{ 0.446} & \multicolumn{2}{c@{}}{12.5 } \\
    \bottomrule
  \end{tabular}
  \vspace{-2ex}
\end{table}

\subsection{In-depth Analysis of CoST and Sketching}
Figure~\ref{fig:ga-gf-cost} tracks accuracy across CoST iterations, i.e., self-training steps with progressively increasing pseudo-label ratio, for the GA and GF pathways, and the fused final prediction. GA consistently outperforms GF while final prediction exceeds both, and performance gaps narrow over iterations as joint optimization on pseudo-labels enables bidirectional knowledge transfer. All curves rise rapidly, then plateau or decline, with optimal iterations varying across datasets. We thus iterate until unlabeled nodes are exhausted and select the checkpoint with the highest validation performance.

Here we define the overall purity of clusters in sketching as $\frac{\sum_{j=1}^{N'} \max_{y \in \mathcal{Y}} |\{v_i \in \mathcal{C}_j : y_i = y\}|}{N}$, which measures the extent to which clusters contain nodes of a single class. As shown in Figure~\ref{fig:purity_comparison}, reassignment consistently improves purity,e.g., {\em WikiCS} exhibits the largest relative improvement (from 0.64 to 0.79), indicating that reassignment is effective for datasets with initially noisy cluster structures. These results validate that our reassignment after clustering successfully refines cluster quality, improving sketching.

% \stitle{Original vs.\ Synthetic Attributes}
% Figure~\ref{fig:tsne_comparison}: t-SNE of node attributes. Originals show entangled classes; synthetic sets form well-separated clusters, indicating that distillation preserves and sharpens class-discriminative structure despite heavy compression.

%visuilization-cst-dual path
\begin{figure}[!t]
\centering
\vspace{1mm}
\begin{tikzpicture}[baseline]
    \draw[line width=0.6mm, 
          mark=triangle,  
          color=orange, 
          mark size=1.5pt,  
          mark repeat=0,    
          mark phase=0] 
    (0,0) -- (0.5cm,0) node[right,font=\footnotesize] {$\PM^{\text{\tiny GA}}$};
    
    \draw[line width=0.6mm, 
          mark=square, 
          color=teal, 
          mark size=1.5pt,
          mark repeat=0,
          mark phase=0] 
    (1.5cm,0) -- (2cm,0) node[right,font=\footnotesize] {$\PM^{\text{\tiny GF}}$};
    
    \draw[line width=0.6mm, 
          mark=star,     
          color=mypurple, 
          mark size=1.5pt,
          mark repeat=0,
          mark phase=0] 
    (3.5cm,0) -- (4cm,0) node[right,font=\footnotesize] {$\PM$};
\end{tikzpicture}
\vspace{-2ex}

\subfloat[\em Cora]{
\begin{tikzpicture}[trim axis left,trim axis right]
\begin{axis}[
    width=0.5\columnwidth, height=0.4\columnwidth, 
    axis y line*=left, axis x line*=bottom,
    xlabel style={font=\footnotesize},
    ylabel={Acc(\%)},ylabel style={font=\footnotesize},
    yticklabel style={font=\scriptsize},
    xtick={0,1,2,3,4,5,6,7,8,9},
    xticklabel style={font=\tiny},
    ymin=70, ymax=83,
    every mark/.append style={mark size=0.5pt, fill=none},
    scaled x ticks=false,
    xtick scale label code/.code={}
]
\addplot[line width=0.6mm,mark=none,color=orange] coordinates {
    (0, 75.1) (1, 76.1) (2,78.3) (3, 80.3) (4,81.9 ) 
    (5,79.9) (6,82.4 ) (7,82.0 ) (8,82.4 ) (9,82.3 )
};

\addplot[line width=0.6mm,mark=none,color=teal] coordinates {
    (0, 70.8) (1, 76.3) (2,78.3 ) (3, 78.7) (4,79.2 ) 
    (5,79.1 ) (6, 80.5) (7, 80.3) (8,80.3 ) (9, 80.4)
};

\addplot[line width=0.6mm,mark=none,color=mypurple] coordinates {
    (0, 78.1) (1, 79.0) (2, 80.1) (3, 80.1) (4, 82.7) 
    (5, 82.7) (6, 82.2) (7, 82.2) (8, 82.5) (9,82.4)
};

\end{axis}
\end{tikzpicture}
}\hspace{15mm}  
\subfloat[\em History]{
\begin{tikzpicture}[trim axis left,trim axis right]
\begin{axis}[
    width=0.5\columnwidth, height=0.4\columnwidth,
    axis y line*=left, axis x line*=bottom,
    xlabel style={font=\footnotesize},
    ylabel={Acc(\%)},ylabel style={font=\footnotesize},
    yticklabel style={font=\scriptsize},
    xtick={0,1,2,3,4,5,6,7,8,9},
    xticklabel style={font=\tiny},
    ymin=60, ymax=78,
    every mark/.append style={mark size=0.5pt, fill=none},
    scaled x ticks=false,
    xtick scale label code/.code={}
]
\addplot[line width=0.6mm,mark=none,color=orange] coordinates {
    (0,68.8) (1, 72.8) (2,76.9 ) (3, 76.1) (4, 73.9) 
    (5,71.4 ) (6, 71.2 ) (7,71.4) (8,71.3) (9, 71.3 )
};

\addplot[line width=0.6mm,mark=none,color=teal] coordinates {
    (0,62.9 ) (1, 71.4) (2, 73.9 ) (3, 71.7) (4, 71.6) 
    (5, 71.4) (6, 70.9) (7, 70.9) (8, 71.2) (9, 71.1 )
};

\addplot[line width=0.6mm,mark=none,color=mypurple] coordinates {
    (0, 73.0) (1, 74.3) (2,77.9 ) (3, 76.8) (4, 74.2) 
    (5, 71.9) (6, 71.9) (7,71.8 ) (8, 71.7) (9,72.1)
};

\end{axis}
\end{tikzpicture}
}
\vspace{-2ex}
\caption{Accuracy by $\PM$, $\PM^{\text{\tiny GA}}$, and $\PM^{\text{\tiny GF}}$ when varying \#iterations.}
\label{fig:ga-gf-cost}
\vspace{-2ex}
\end{figure}
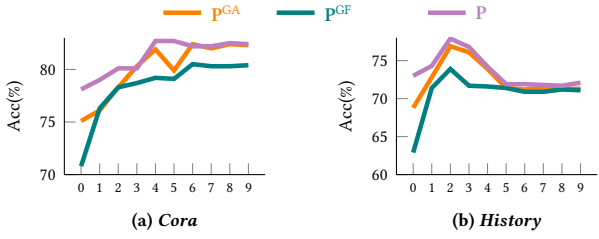

\begin{figure}[!t]
  \centering
  \begin{tikzpicture}
    \begin{axis}[
        width=6.0cm, height=3.5cm,
        axis lines=left,
        ymin=0.6, ymax=0.85,
        ylabel={Purity Value},
        ylabel style={font=\footnotesize, label distance=0.2cm},
        xtick=data,
        xticklabels={{\em Cora}, {\em History}, {\em Photo}, {\em WikiCS}},
        xticklabel style={font=\footnotesize},
        ytick={0.6, 0.65, 0.7, 0.75, 0.8, 0.85},
        yticklabel style={font=\small},
        bar width=0.3cm,
        legend pos=outer north east,
        legend style={
          font=\small,          
          fill=none,           
          draw=none,           
          inner sep=1pt,      
          outer sep=0pt,        
          anchor=north,
          at={(0.5,1.08)},    
          legend columns=2
        },
        legend image code/.code={
          \draw[#1, draw=none] (0cm,-0.08cm) rectangle (0.15cm,0.08cm);
        },
        symbolic x coords={{\em Cora}, {\em History}, {\em Photo}, {\em WikiCS}},
        enlarge x limits=0.2,
      ]
      
      \addplot[
        ybar,
        fill=greenacc, 
        draw=greenacc,
        bar shift=-0.2cm,
        draw=none  
      ] coordinates {
        ({\em Cora}, 0.8128)  
        ({\em History}, 0.7673)
        ({\em Photo}, 0.7673)
        ({\em WikiCS}, 0.6378)
      };

      \addplot[
        ybar,
        fill=myorange, 
        draw=myorange,
        bar shift=0.2cm,
        draw=none 
      ] coordinates {
        ({\em Cora}, 0.8364)
        ({\em History}, 0.8027)
        ({\em Photo}, 0.8027)
        ({\em WikiCS}, 0.7911)
      };
      
      \legend{Before, After}
    \end{axis}
    \end{tikzpicture}
    \vspace{-2ex}
  \caption{Cluster purity before/after greedy reassignment.}
  \label{fig:purity_comparison}
\end{figure}
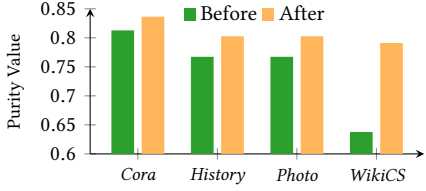

\subsection{Hyperparameter Analysis}
\label{exp:hyperpara}

Figure~\ref{fig:hyper} examines four key hyperparameters. Increasing input keywords $\xi$ from 64 to 512 improves \textit{Cora} gradually to 74.5\%, while \textit{History} remains stable around 73.3\% and then drops to 70.6\% at 512, suggesting excessive keywords may introduce noise. Output tokens limits $\kappa$ show similar divergence: \textit{Cora} improves to 73.9\% at $\kappa=256$  while \textit{History} at 73.4\% with $\kappa=12 $. For the number of candidate summaries per cluster $q$. Our method is robust on \textit{Cora} with 74.0$\pm$0.2\% across $q \in \{2,3,4,5\}$. \textit{History} drops to 70.4\% at $q=4$ but recovers at $q=5$; 
For number of candidate synthetic graphs $n_s$,\textit{Cora} maintains 72.8\%--74.1\% across $\{10, 50, 100, 200\}$. \textit{History} improves from 72.9\% to 75.5\%, demonstrating larger pools benefit complex textual datasets. 
% hyper
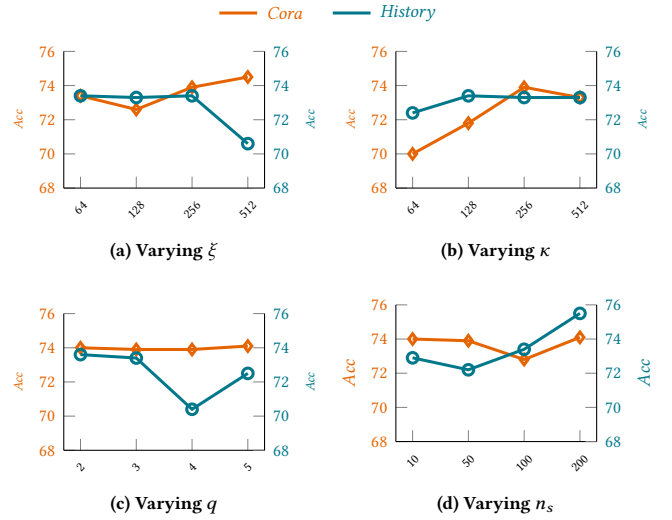
\begin{figure}[t]
\centering
\vspace{1mm}
\begin{tikzpicture}[baseline]
    \draw[line width=0.4mm,mark=diamond,color=myorange-new,mark size=2.5pt] (0,0) -- (0.5cm,0) node[right,font=\footnotesize] {\em Cora};
    \draw[line width=0.4mm,mark=o,color=mygreen-new,mark size=2.5pt] (1.5cm,0) -- (2cm,0) node[right,font=\footnotesize] {\em History};
\end{tikzpicture}

\vspace{-2ex}
% Subplot 1: Varying $\xi$
\subfloat[Varying $\xi$]{
\begin{tikzpicture}[trim axis left,trim axis right]
\begin{axis}[
    width=0.5\columnwidth, height=\columnwidth/2.5,
    axis y line*=left, axis x line*=bottom,
    ylabel={\it Acc},ylabel style={font=\tiny,color=myorange-new},
    yticklabel style={font=\scriptsize,color=myorange-new},
    symbolic x coords={64,128,256,512},
    xtick=data,
    xticklabel style={font=\tiny,rotate=45},
    ymin=68, ymax=76,
    every mark/.append style={mark size=2.5pt},
    scaled x ticks=false,
    xtick scale label code/.code={},
]
\addplot[line width=0.4mm,mark=diamond,color=myorange-new] coordinates {
    (64,73.4) (128,72.6) (256,73.9) (512,74.5)
};
\end{axis}
\begin{axis}[
    width=0.5\columnwidth, height=\columnwidth/2.5,
    axis y line*=right, axis x line=none,
    ylabel={\it Acc},ylabel style={font=\tiny,color=mygreen-new},
    yticklabel style={font=\scriptsize,color=mygreen-new},
    symbolic x coords={64,128,256,512},
    xtick=data,
    xticklabel=\empty,
    ymin=68, ymax=76,
    every mark/.append style={mark size=2.5pt},
    scaled x ticks=false,
    xtick scale label code/.code={},
]
\addplot[line width=0.4mm,mark=o,color=mygreen-new] coordinates {
    (64,73.4) (128,73.3) (256,73.4) (512,70.6)
};
\end{axis}
\end{tikzpicture}
}\hspace{15mm}  % Horizontal spacing between subplots
\subfloat[Varying $\kappa$]{
\begin{tikzpicture}[trim axis left,trim axis right]
\begin{axis}[
    width=0.5\columnwidth, height=\columnwidth/2.5,
    axis y line*=left, axis x line*=bottom,
    ylabel={\it Acc},ylabel style={font=\tiny,color=myorange-new},
    yticklabel style={font=\scriptsize,color=myorange-new},
    symbolic x coords={64,128,256,512},
    xtick=data,
    xticklabel style={font=\tiny,rotate=45},
    ymin=68, ymax=76,
    every mark/.append style={mark size=2.5pt},
    scaled x ticks=false,
    xtick scale label code/.code={},
]
\addplot[line width=0.4mm,mark=diamond,color=myorange-new] coordinates {
    (64,70.0) (128,71.8) (256,73.9) (512,73.3)
};
\end{axis}
\begin{axis}[
    width=0.5\columnwidth, height=\columnwidth/2.5,
    axis y line*=right, axis x line=none,
    ylabel={\it Acc},ylabel style={font=\tiny,color=mygreen-new},
    yticklabel style={font=\scriptsize,color=mygreen-new},
    symbolic x coords={64,128,256,512},
    xtick=data,
    xticklabel=\empty,
    ymin=68, ymax=76,
    every mark/.append style={mark size=2.5pt},
    scaled x ticks=false,
    xtick scale label code/.code={},
]
\addplot[line width=0.4mm,mark=o,color=mygreen-new] coordinates {
    (64,72.4) (128,73.4) (256,73.3) (512,73.3)
};
\end{axis}
\end{tikzpicture}
}

\subfloat[Varying $q$]{
\begin{tikzpicture}[trim axis left,trim axis right]
\begin{axis}[
    width=0.5\columnwidth, height=\columnwidth/2.5,
    axis y line*=left, axis x line*=bottom,
    ylabel={\it Acc},ylabel style={font=\tiny,color=myorange-new},
    yticklabel style={font=\scriptsize,color=myorange-new},
    symbolic x coords={2,3,4,5},
    xtick=data,
    xticklabel style={font=\tiny,rotate=45},
    ymin=68, ymax=76,
    every mark/.append style={mark size=2.5pt},
    scaled x ticks=false,
    xtick scale label code/.code={},
]
\addplot[line width=0.4mm,mark=diamond,color=myorange-new] coordinates {
    (2,74.0) (3,73.9) (4,73.9) (5,74.1)
};
\end{axis}
\begin{axis}[
    width=0.5\columnwidth, height=\columnwidth/2.5,
    axis y line*=right, axis x line=none,
    ylabel={\it Acc},ylabel style={font=\tiny,color=mygreen-new},
    yticklabel style={font=\scriptsize,color=mygreen-new},
    symbolic x coords={2,3,4,5},
    xtick=data,
    xticklabel=\empty,
    ymin=68, ymax=76,
    every mark/.append style={mark size=2.5pt},
    scaled x ticks=false,
    xtick scale label code/.code={},
]
\addplot[line width=0.4mm,mark=o,color=mygreen-new] coordinates {
    (2,73.6) (3,73.4) (4,70.4) (5,72.5)
};
\end{axis}
\end{tikzpicture}
}  % Horizontal spacing between subplots
\hspace{15mm}
\subfloat[Varying $n_s$]{
\begin{tikzpicture}[trim axis left,trim axis right]
\begin{axis}[
    width=0.5\columnwidth, height=\columnwidth/2.5,
    axis y line*=left, axis x line*=bottom,
    ylabel={\it Acc},ylabel style={font=\footnotesize,color=myorange-new},
    yticklabel style={font=\scriptsize,color=myorange-new},
    symbolic x coords={10,50,100,200},
    xtick=data,
    xticklabel style={font=\tiny,rotate=45},
    ymin=68, ymax=76,
    every mark/.append style={mark size=2.5pt},
    scaled x ticks=false,
    xtick scale label code/.code={},
]
\addplot[line width=0.4mm,mark=diamond,color=myorange-new] coordinates {
    (10,74.0) (50,73.9) (100,72.8) (200,74.1)
};
\end{axis}
\begin{axis}[
    width=0.5\columnwidth, height=\columnwidth/2.5,
    axis y line*=right, axis x line=none,
    ylabel={\it Acc},ylabel style={font=\footnotesize,color=mygreen-new},
    yticklabel style={font=\scriptsize,color=mygreen-new},
    symbolic x coords={10,50,100,200},
    xtick=data,
    xticklabel=\empty,
    ymin=68, ymax=76,
    every mark/.append style={mark size=2.5pt},
    scaled x ticks=false,
    xtick scale label code/.code={},
]
\addplot[line width=0.4mm,mark=o,color=mygreen-new] coordinates {
    (10,72.9) (50,72.2) (100,73.4) (200,75.5)
};
\end{axis}
\end{tikzpicture}
}
\caption{Hyperparameter analysis: effect of $\xi$, $\kappa$, $q$, and $n_s$ on accuracy (mean~\%) $\pm$ std.).}
\label{fig:hyper}
% Shared legend for all subplots

\end{figure}
\subsection{Visualization}
\label{exp:vis}
As shown in Figure~\ref{fig:tsne_comparison}, t-SNE visualizations of node attributes reveal that original data features exhibit entangled class boundaries, while synthetic features generated by our distillation method form clearly separated clusters. This result confirms that our distillation strategy not only preserves but also sharpens class-discriminative structural information, even with substantial compression ratios.

\usetikzlibrary{matrix, positioning}

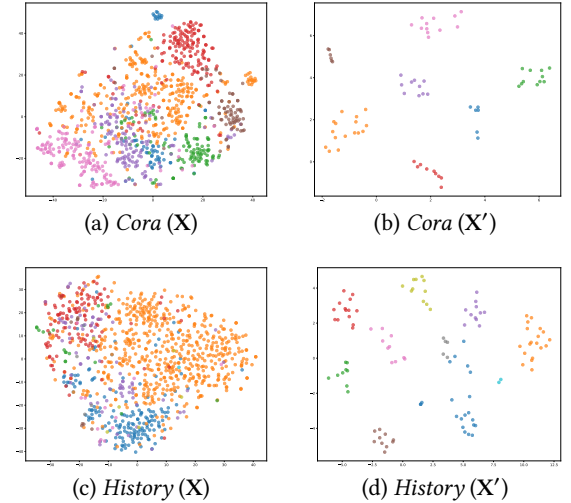
\begin{figure}[htbp]
    \centering
    \begin{tikzpicture}[
        node distance=0.5cm,
        every node/.style={inner sep=0pt}
    ]

        \matrix (m) [matrix of nodes, column sep=0.5cm, row sep=0.8cm] {
            \includegraphics[width=0.4\linewidth]{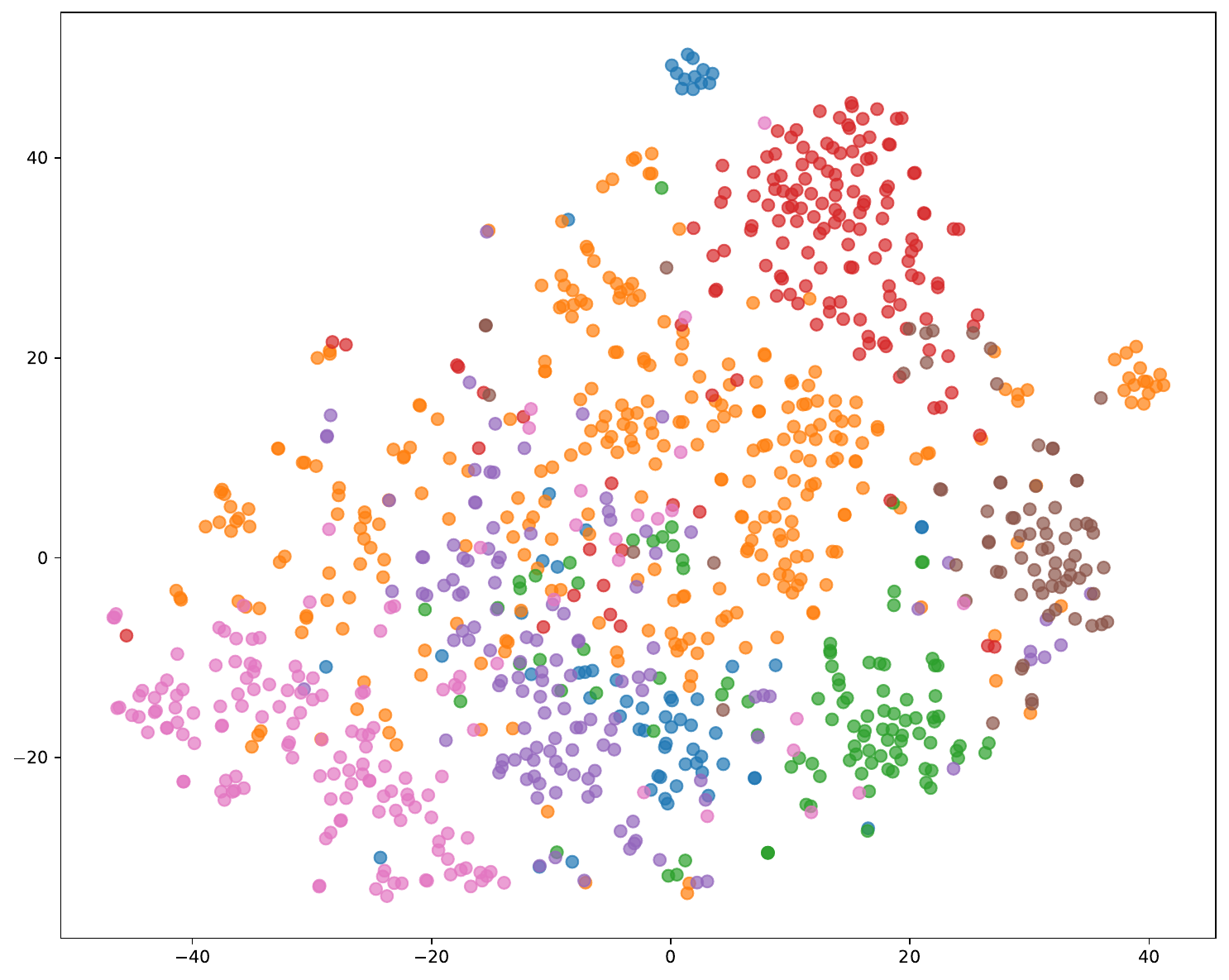} &
            \includegraphics[width=0.4\linewidth]{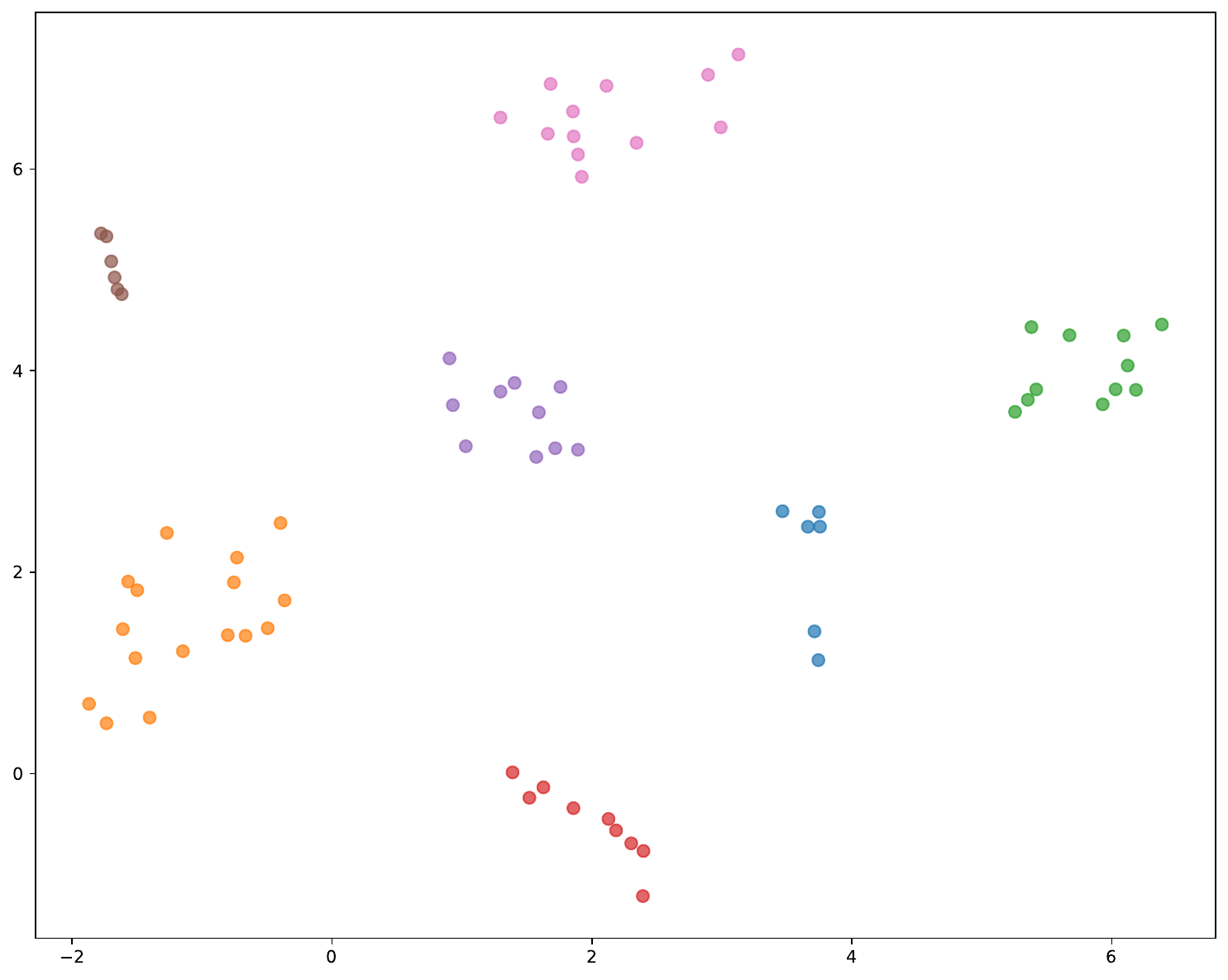} \\
            \includegraphics[width=0.4\linewidth]{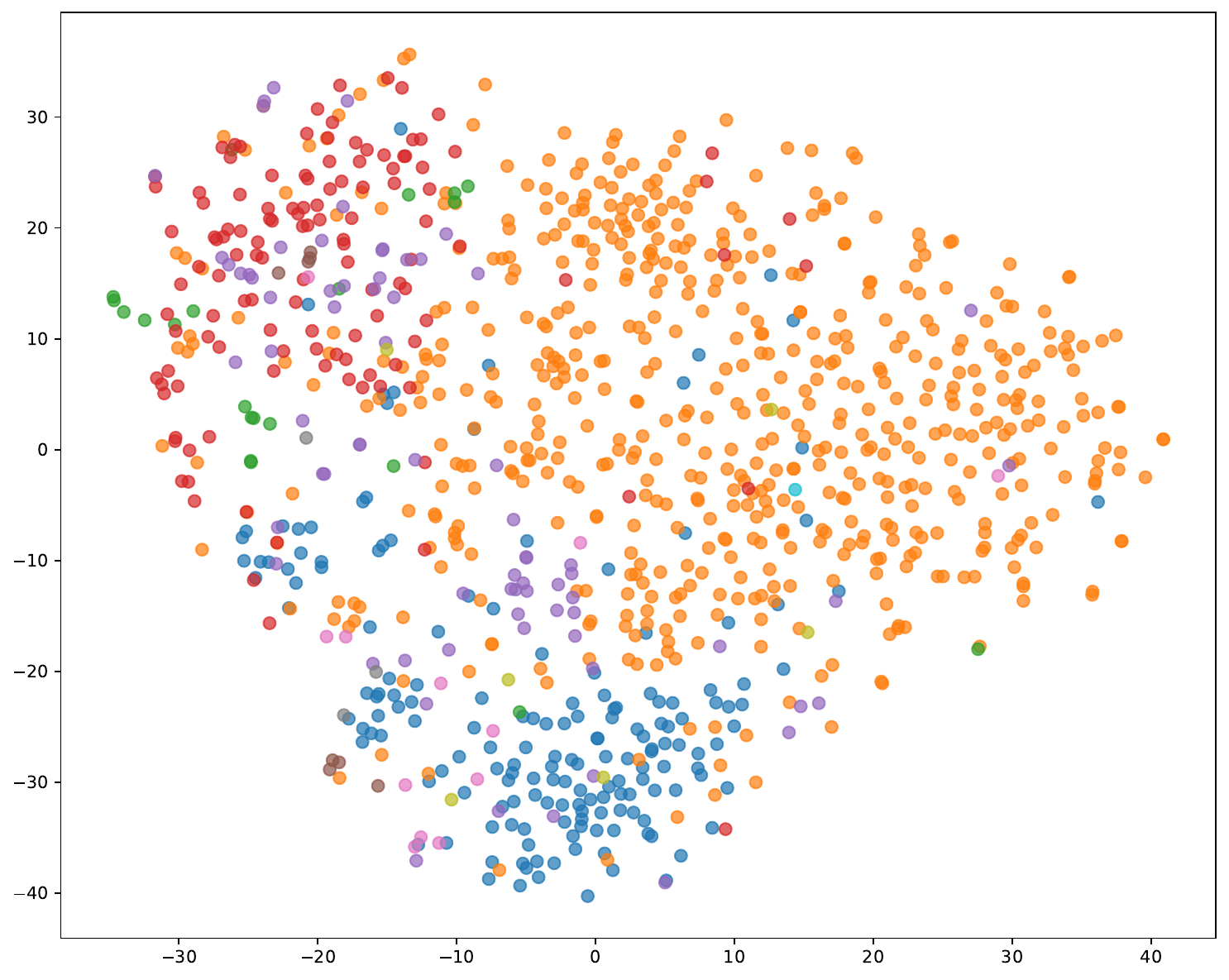} &
            \includegraphics[width=0.4\linewidth]{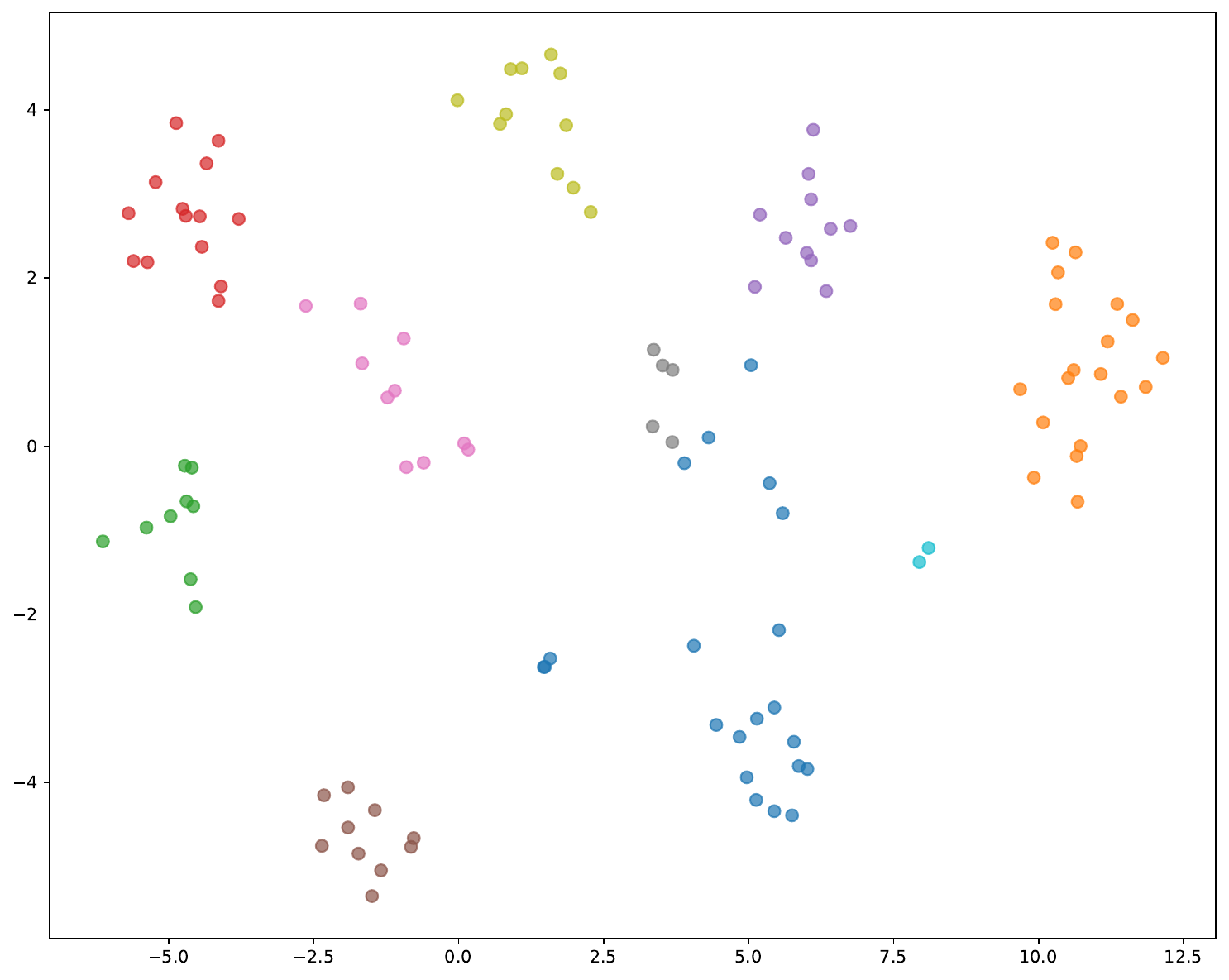} \\
        };
        
        \node[below=0.1cm of m-1-1] {(a) {\em Cora} ($\XM$)};

        \node[below=0.1cm of m-1-2] {(b) {\em Cora} ($\XM'$)};

        \node[below=0.1cm of m-2-1] {(c) {\em History} ($\XM$)};

        \node[below=0.1cm of m-2-2] {(d) {\em History} ($\XM'$)};

        \node[below=0.1cm of m-1-1, opacity=0] (label1) {a};
        \node[below=0.1cm of m-1-2, opacity=0] (label2) {b};
        \node[below=0.1cm of m-2-1, opacity=0] (label3) {c};
        \node[below=0.1cm of m-2-2, opacity=0] (label4) {d};
        
    \end{tikzpicture}
    
    \caption{t-SNE of node attributes: original ($\XM$) vs.\ synthetic ($\XM'$) on {\em Cora} and {\em History}.}
    \label{fig:tsne_comparison}
\end{figure}

\subsection{Conclusion}
\label{sec:conclusion}

This paper proposes \algo{} for semi-supervised TAG distillation. Guided by WSD, \algo{} introduces three modules: graph-text collaborative encoding, WSD-based graph sketching, and keywords-based LLM text synthesis. \algo{} achieves the state-of-the-art performance compared to current graph and text distillation methods on various real-world datasets. The synthetic TAGs obtained from \algo{} enable downstream LLM4Graph methods to be efficient and effective. A potential direction for future work is to explore multi-modal graph distillation for continual learning.
\section{Acknowledgments}
Renchi Yang is supported by the Guangdong and Hong Kong Universities ''1+1+1'' Joint Research Collaboration Scheme, project No.: 2025A0505000002, the NSFC (No. 62302414), the Hong Kong RGC ECS grant (No. 22202623), and YCRG (No. C2003-23Y). Tsz Nam Chan is supported by the Natural Science Foundation of China under grants 23IAA00610 and 62572326.

\pagebreak
\bibliographystyle{acm}
\balance
\bibliography{acm_ref}

% \pagebreak
\appendix
\section{Complete Algorithm and Analysis}
\label{sec:algorithm_complexity}

This section provides a formal algorithmic description of the \algo{} framework along with a detailed complexity analysis of its three core modules.

\stitle{Algorithm Overview}
Algorithms~\ref{alg:encoding}, \ref{alg:sketching}, and~\ref{alg:synthesis} present the complete \algo{} pipeline, which takes a large-scale TAG $\G$ as input and produces a compact synthetic TAG $\G'$ with human-readable text attributes.

\stitle{Notation}
We follow Section~\ref{sec:notations} for basic TAG notations: $N = |\V|$, $M = |\EDG|$, $D$ the feature dimension, $K$ the number of classes. For Module 1 (\S\ref{sec:feat-encode}): dual-pathway encoders $\textsf{GCN}$ and $\textsf{MLP}$, hidden dimension $h$, attention parameters $\WM_{\text{att}} \in \mathbb{R}^{h}$ and $\mathbf{b}_{\text{att}} \in \mathbb{R}$, soft labels $\PM^{\text{\tiny GA}}$, $\PM^{\text{\tiny GF}}$, $\PM$, confidence scores $\max_{1\le k\le K}\PM^{\text{\tiny GA}}_{i,k}$, pseudo-label ratio $\rho$ with increment step $0.1$, consensus set $\mathcal{R}^{\text{\tiny CS}}$ and disagreement subsets $\mathcal{R}^{\text{\tiny DS}}_{\text{\tiny GA}}$, $\mathcal{R}^{\text{\tiny DS}}_{\text{\tiny GF}}$, labeled sets $\mathcal{U}$, $\mathcal{U}^{\text{\tiny GA}}$, $\mathcal{U}^{\text{\tiny GF}}$. For Module 2 (\S\ref{sec:sketch}): target condensed size $N'$, sketching matrix $\SM \in \mathbb{R}^{N \times N'}$, clusters $\{\C_1,\ldots,\C_{N'}\}$, K-Means iterations. For Module 3 (\S\ref{sec:text-synthesis}): keyword count $\xi$, intra-cluster TF-IDF, keywords $\mathcal{W}_i$, candidate count $q$ (typically $q=3$), prompt $\mathcal{P}$, sample count $n_s$ (typically $n_s=100$), top-$k$ validation (often $k=5$), text encoder (SBERT).

\begin{algorithm}[htb]
\caption{Graph-Text Collaborative Encoding}
\label{alg:encoding}
% \DontPrintSemicolon
\KwIn{TAG $\G$ with $\AM$, $\XM$, $\YM$, and initial labeled set $\V_{\text{tr}}$.}
\KwOut{Soft labels $\PM$ and labeled sets $\mathcal{U}$, $\mathcal{U}^{\text{\tiny GA}}$, $\mathcal{U}^{\text{\tiny GF}}$.}
\vspace{0.5ex}
\textbf{Step 1 --- Dual-pathway encoding.}\;
$\HM^{\text{\tiny GA}} \gets \textsf{GCN}(\AM, \XM)$; $\HM^{\text{\tiny GF}} \gets \textsf{MLP}(\XM)$\;
$\PM^{\text{\tiny GA}} \gets \textsf{softmax}(\HM^{\text{\tiny GA}})$; $\PM^{\text{\tiny GF}} \gets \textsf{softmax}(\HM^{\text{\tiny GF}})$\;
\For{$i = 1$ \KwTo $N$}{
    $\alpha^{\text{\tiny GA}}_i \gets \sigma(\WM_{\text{att}} \cdot \HM_i^{\text{\tiny GA}} + \mathbf{b}_{\text{att}})$; $\alpha^{\text{\tiny GF}}_i \gets 1-\alpha^{\text{\tiny GA}}_i$\;
    $\PM_i \gets \textsf{softmax}(\alpha^{\text{\tiny GA}}_i \odot \HM_i^{\text{\tiny GA}} + \alpha^{\text{\tiny GF}}_i \odot \HM_i^{\text{\tiny GF}})$\;
}
\vspace{0.5ex}
\textbf{Step 2 --- Collaborative self-training.}\;
Initialize $\mathcal{U} \gets \V_{\text{tr}}$, $\mathcal{U}^{\text{\tiny GA}} \gets \V_{\text{tr}}$, $\mathcal{U}^{\text{\tiny GF}} \gets \V_{\text{tr}}$; $\rho \gets 0.1$\;
\While{$\rho \le 1.0$}{
    Recompute $\HM^{\text{\tiny GA}}$, $\HM^{\text{\tiny GF}}$, $\PM^{\text{\tiny GA}}$, $\PM^{\text{\tiny GF}}$, $\PM$\;
    $\mathcal{R} \gets$ top-$\rho$ proportion of unlabeled nodes $\V \setminus \mathcal{U}$ by confidence $\max_{k}\PM^{\text{\tiny GA}}_{i,k}$\;
    $\mathcal{R}^{\text{\tiny CS}} \gets \{v_i \in \mathcal{R} \mid \argmax{k}\PM^{\text{\tiny GA}}_{i,k} = \argmax{k}\PM^{\text{\tiny GF}}_{i,k}\}$\;
    $\mathcal{R}^{\text{\tiny DS}} \gets \mathcal{R} \setminus \mathcal{R}^{\text{\tiny CS}}$\;
    $\mathcal{R}^{\text{\tiny DS}}_{\text{\tiny GA}} \gets \{v_i \in \mathcal{R}^{\text{\tiny DS}} \mid \max_{k}\PM^{\text{\tiny GA}}_{i,k} > \max_{k}\PM^{\text{\tiny GF}}_{i,k}\}$\;
    $\mathcal{R}^{\text{\tiny DS}}_{\text{\tiny GF}} \gets \mathcal{R}^{\text{\tiny DS}} \setminus \mathcal{R}^{\text{\tiny DS}}_{\text{\tiny GA}}$\;
    $\mathcal{U} \gets \mathcal{U} \cup \mathcal{R}^{\text{\tiny CS}} \cup \mathcal{R}^{\text{\tiny DS}}_{\text{\tiny GA}} \cup \mathcal{R}^{\text{\tiny DS}}_{\text{\tiny GF}}$\;
    $\mathcal{U}^{\text{\tiny GA}} \gets \mathcal{U}^{\text{\tiny GA}} \cup \mathcal{R}^{\text{\tiny CS}} \cup \mathcal{R}^{\text{\tiny DS}}_{\text{\tiny GA}}$\;
    $\mathcal{U}^{\text{\tiny GF}} \gets \mathcal{U}^{\text{\tiny GF}} \cup \mathcal{R}^{\text{\tiny CS}} \cup \mathcal{R}^{\text{\tiny DS}}_{\text{\tiny GF}}$\;
    Update $\WM_{\text{att}}$, $\textsf{GCN}$, $\textsf{MLP}$ by minimizing $\mathcal{L} = \textsf{CE}(\PM, \hat{\YM}) + \textsf{CE}(\PM^{\text{\tiny GF}},\hat{\YM}^{\text{\tiny GF}}) + \textsf{CE}(\PM^{\text{\tiny GA}},\hat{\YM}^{\text{\tiny GA}})$\;
    $\rho \gets \rho + 0.1$\;
}
\Return $\PM$, $\mathcal{U}$, $\mathcal{U}^{\text{\tiny GA}}$, $\mathcal{U}^{\text{\tiny GF}}$\;
\end{algorithm}

\stitle{Complexity Analysis}
We analyze the time complexity of each module. Let $N$ be the number of nodes, $M$ the number of edges, $D$ the feature dimension of $\XM$, $L$ the number of GCN layers, $h$ the hidden dimension, $K$ the number of classes, $T = 10$ the number of self-training epochs ($\rho$ from $0.1$ to $1.0$ with step $0.1$), $I_{\text{km}}$ the number of K-Means iterations, $W$ the average text length, $V$ the total vocabulary size (TF-IDF), $\mathcal{T}_{\text{LLM}}$ the cost of one LLM call, $I_{\text{sink}}$ the number of Sinkhorn iterations for WSD computation, and $C$ the size of the sampled node subset used in WSD ($N' < C < N$).

\textbf{Module 1: Graph-Text Collaborative Encoding.}
The GA pathway (GCN with $L$ layers) costs $\mathcal{O}(L(MD + Nh^2))$ for sparse matrix multiplication and transformations. The GF pathway (MLP) costs $\mathcal{O}(NDh)$. Attention fusion requires $\mathcal{O}(Nh)$. Collaborative self-training repeats for $T$ iterations, each involving confidence selection $\mathcal{O}(N \log N)$, consensus verification $\mathcal{O}(N)$, and backpropagation $\mathcal{O}(L(MD + Nh^2))$. Total for Module 1:
\begin{equation}
\mathcal{T}_{\text{M1}} = \mathcal{O}(T \cdot L(MD + Nh^2) + TN \log N)
\end{equation}

\textbf{Module 2: WSD-based Graph Sketching.}
K-Means clustering on $N$ soft label vectors of dimension $K$ into $N'$ clusters costs $\mathcal{O}(N \cdot N' \cdot K \cdot I_{\text{km}})$. Greedy reassignment involves computing class-cluster affinities $\mathcal{O}(NK)$ and distance computations $\mathcal{O}(NKN')$. Graph compression via sparse matrix operations costs $\mathcal{O}(ND + M + N'^2)$. Total for Module 2:
\begin{equation}
\mathcal{T}_{\text{M2}} = \mathcal{O}(NN'K \cdot I_{\text{km}} + ND + M + N'^2)
\end{equation}

\begin{algorithm}[htb]
\caption{WSD-based Graph Sketching}
\label{alg:sketching}
% \DontPrintSemicolon
\KwIn{Soft labels $\PM$, $\AM$, $\XM$, $\YM$, and target size $N'$.}
\KwOut{Sketching matrix $\SM$ and condensed $\AM'$, $\XM'$, $\YM'$.}
\vspace{0.5ex}
\textbf{Step 1 --- $K$-Means clustering on soft labels.}\;
$\{\C_1,\C_2,\ldots,\C_{N'}\} \gets \textsf{K-Means}(\PM, N')$\;
Construct $\SM$ s.t. $\SM_{k,i} = \frac{1}{|\C_k|}$ if $v_i \in \C_k$, else $0$\;
\vspace{0.5ex}
\textbf{Step 2 --- Greedy reassignment post-processing.}\;
\For{$k=1$ \KwTo $K$}{
    \For{$j=1$ \KwTo $N'$}{
        Compute affinity $\frac{|\{v_i\in \C_j|\ \YM_{i,k}=1\}|}{|\{v_i\in \V|\ \YM_{i,k}=1\}|}$\;
    }
}
% Select top-$N'$ class-cluster pairs with the highest affinities; assign nodes in these classes to corresponding clusters\;
$\{y_{k_i},\C_{k_i}\}_{i=1}^{N'}\gets$ the pairs with the highest affinity values\;
Assign nodes with label $y_{k_i}$ to $\C_{k_i}\ \forall{1\le i\le N'}$\;
\For{each node $v_i$ w/o labels $\{y_{k_i}\}_{i=1}^{N'}$}{
    Add $v_i$ to $\C_k$ s.t. $\min_{1\le k\le N^\prime}\|\PM_i-\textsf{Mean}(\C_k)\|_2$\;
}
Update $\SM$ accordingly\;
\vspace{0.5ex}
\textbf{Step 3 --- Graph compression.}\;
$\AM' \gets \SM^\top \AM \SM$; $\XM' \gets \SM^\top \XM$\;
$\YM'_{i,k} \gets \begin{cases} 1 & \text{if } k = \argmax{1\le k'\le K}{(\SM^\top \PM)}_{i,k'} \\ 0 & \text{otherwise} \end{cases}$\;
\Return $\SM$, $\AM'$, $\XM'$, $\YM'$\;
\end{algorithm}

\textbf{Module 3: Keywords-based Text Synthesis.}
TF-IDF computation across $N'$ clusters with total vocabulary size $V$ costs $\mathcal{O}(VW \cdot N' + N \cdot W)$. SBERT-based keyword filtering costs $\mathcal{O}(N'\xi D)$. LLM generation for $N'$ nodes with $q$ candidates each costs $\mathcal{O}(N'q \cdot \mathcal{T}_{\text{LLM}})$. WSD computation for $n_s$ candidate graphs using Sinkhorn algorithm costs $\mathcal{O}(n_s(N + N')^2 I_{\text{sink}})$. In practice, sampling a subset of original nodes with size $C$ ($N'<C<N$) is enough, so that the Sinkhorn algorithm can be $\mathcal{O}(n_s(C + N')^2 I_{\text{sink}})$. Total for Module 3:
\begin{equation}
\mathcal{T}_{\text{M3}} = \mathcal{O}(VWN' + N'\xi D + N'q \cdot \mathcal{T}_{\text{LLM}} + n_s(C + N')^2 I_{\text{sink}})
\end{equation}

\textbf{Overall Complexity.}
The total time complexity is $\mathcal{T}_{\algo{}} = \mathcal{T}_{\text{M1}} + \mathcal{T}_{\text{M2}} + \mathcal{T}_{\text{M3}}$. The space complexity is dominated by storing the adjacency matrix ($\mathcal{O}(M)$), feature matrices ($\mathcal{O}(ND + N'h)$), soft labels ($\mathcal{O}(NK)$), and candidate texts ($\mathcal{O}(N'qW)$), yielding $\mathcal{S}_{\algo{}} = \mathcal{O}(M + ND + N'h + NK + N'qW)$.

In practice, since $N' \ll N$, Module 3 operates efficiently despite LLM calls. The dominant cost in Module 1 is the GCN operations scaling linearly with edges $M$. Module 2's clustering is efficient when $N'$ is small. The WSD validation in Module 3 uses a small $n_s$ (e.g., $100$) to balance quality and cost.

% \begin{table}[!t]
% \centering
% \small
% \caption{Time complexity breakdown of \algo{}.}
% \label{tab:complexity}
% \vspace{-2ex}
% \small
% \begin{tabular}{@{}ll@{}}
% \toprule
% \textbf{Component} & \textbf{Time Complexity} \\
% \midrule
% Dual-Pathway Encoders (per epoch) & $\mathcal{O}(L(MD + Nh^2))$ \\
% Collaborative Self-Training ($T$ iterations) & $\mathcal{O}(T(L(MD + Nh^2) + N \log N))$ \\
% K-Means Clustering & $\mathcal{O}(NN'K \cdot I_{\text{km}})$ \\
% Greedy Reassignment & $\mathcal{O}(NKN')$ \\
% Graph Sketching & $\mathcal{O}(ND + M + N'^2)$ \\
% Intra-cluster TF-IDF & $\mathcal{O}(VWN' + NW)$ \\
% SBERT Keyword Filtering & $\mathcal{O}(N'\xi D)$ \\
% LLM Generation & $\mathcal{O}(N'q \cdot \mathcal{T}_{\text{LLM}})$ \\
% WSD-based Selection & $\mathcal{O}(n_s(C+N')^2 I_{\text{sink}})$ \\
% \bottomrule
% \end{tabular}
% \end{table}

\begin{algorithm}[htb]
\caption{Keywords-based Text Synthesis with LLMs}
\label{alg:synthesis}
% \DontPrintSemicolon
\KwIn{Sketching matrix $\SM$, clusters $\{\C_1,\ldots,\C_{N'}\}$, text $\mathcal{S}$, labels $\YM'$, integers $\xi$, $q$, $n_s$, $k$, and prompt $\mathcal{P}$.}
\KwOut{Synthetic texts $\mathcal{S}' = \{s'_1,\ldots,s'_{N'}\}$ for $\G'$.}
\vspace{0.5ex}
\textbf{Step 1 --- Cluster-based keyword extraction.}\;
\For{$i = 1$ \KwTo $N'$}{
    $\mathcal{S}_i \gets \{s_j \mid v_j \in \C_i\}$\;
    $\text{TF-IDF}(w, \mathcal{S}_i) = \text{TF}(w, \mathcal{S}_i) \times \text{IDF}(w, \mathcal{S}_i)$ $\forall{w\in \mathcal{S}_i}$\;
    Select top-$2\xi$ words by intra-cluster TF-IDF\;
    Prune $\xi$ words with lowest embedding distance to $\XM'_i$\;
    Obtain keywords $\mathcal{W}_i = \{w_{i,1},\ldots,w_{i,\xi}\}$\;
}
\vspace{0.5ex}
\textbf{Step 2 --- Candidate text generation with LLMs.}\;
\For{$i = 1$ \KwTo $N'$}{
    $y'_i \gets \argmax{k}\YM'_{i,k}$\;
    $\{s'_{i,1},\ldots,s'_{i,q}\} \gets \textsf{LLM}(\mathcal{P},\mathcal{W}_i,y'_i)$\;
}
\vspace{0.5ex}
\textbf{Step 3 --- WSD-validation-guided selection.}\;
\For{$l = 1$ \KwTo $n_s$}{
    $\mathcal{S}'_l \gets \{s'_{i,j_i} \mid i \in [N'], j_i \sim \text{Uniform}([q])\}$\;
    $\XM'_l \gets \text{SBERT}(\mathcal{S}'_l)$\;
    Compute $\text{WSD}_l \gets \textsf{WSD}(\XM, \XM'_l)$\;
}
Select top-$k$ candidates with the lowest WSD\;
% Evaluate downstream task performance on original dataset; 
Choose $\mathcal{S}'$ with the best downstream task performance\;
\Return $\mathcal{S}'$\;
\end{algorithm}

\section{Theoretical Proof}
\label{appendix:theory}

\begin{proof}[Proof of Lemma \ref{lem:WSD-clust}]
We establish this equivalence through the geometric form of the 2-Wasserstein distance in the context of graph sketching.

\textbf{Step 1: Distribution construction.} 
Let $\mathcal{P} = \frac{1}{N}\sum_{i=1}^N \delta_{\MM_i}$ denote the empirical distribution of the original features, where $\delta_{\MM_i}$ is the Dirac delta distribution centered at $\MM_i$. The sketching matrix $\SM \in \mathbb{R}^{N \times N'}$ induces a hard clustering partition $\{\mathcal{C}_k\}_{k=1}^{N'}$ of the node set $[N]$, where $\mathcal{C}_k = \{i \in [N] : \SM_{k,i} \neq 0\}$ represents the set of nodes assigned to the $k$-th condensed node.

By construction of $\SM$ in Algorithm~\ref{alg:sketching}, each row $k$ satisfies $\SM_{k,i} = \frac{1}{|\mathcal{C}_k|}$ if $i \in \mathcal{C}_k$, and $\SM_{k,i} = 0$ otherwise. Consequently, the compressed empirical distribution is $\mathcal{Q} = \sum_{k=1}^{N'} \frac{|\mathcal{C}_k|}{N} \delta_{\boldsymbol{\mu}_k}$, where the centroid $\boldsymbol{\mu}_k = \sum_{j=1}^{N} \SM_{k,j} \MM_j = \frac{1}{|\mathcal{C}_k|}\sum_{j \in \mathcal{C}_k} \MM_j$.

\textbf{Step 2: Voronoi property verification.}
We verify that the partition $\{\mathcal{C}_k\}_{k=1}^{N'}$ satisfies the Voronoi property under the squared Euclidean metric. For any sample $\MM_i \in \mathcal{C}_k$ and any other cluster $k' \neq k$, the K-Means clustering in Algorithm~\ref{alg:sketching} (Step 1) assigns each node to the nearest centroid, ensuring:
\[
\|\MM_i - \boldsymbol{\mu}_k\|_2^2 \leq \|\MM_i - \boldsymbol{\mu}_{k'}\|_2^2.
\]

\textbf{Step 3: Optimal transport plan.}
Consider the transport plan $\gamma^* \in \mathbb{R}^{N \times N'}$ induced by the clustering:
\[
\gamma^*_{ik} = 
\begin{cases} 
\frac{1}{N}, & \text{if } \MM_i \in \mathcal{C}_k, \\
0, & \text{otherwise}.
\end{cases}
\]

This plan is feasible: (i) $\sum_{k=1}^{N'} \gamma^*_{ik} = \frac{1}{N}$ for all $i$ (marginal of $\mathcal{P}$), and (ii) $\sum_{i=1}^{N} \gamma^*_{ik} = \frac{|\mathcal{C}_k|}{N}$ for all $k$ (marginal of $\mathcal{Q}$).

\textbf{Step 4: Optimality via Voronoi property.}
For any feasible transport plan $\gamma$, the Voronoi property yields the pointwise lower bound for each $i$:
\[
\sum_{k=1}^{N'} \gamma_{ik} \|\MM_i - \boldsymbol{\mu}_k\|_2^2 
\geq \sum_{k=1}^{N'} \gamma_{ik} \|\MM_i - \boldsymbol{\mu}_{k(i)}\|_2^2 
= \frac{1}{N}\|\MM_i - \boldsymbol{\mu}_{k(i)}\|_2^2,
\]
where $k(i)$ denotes the cluster index of $\MM_i$. Summing over all $i$ gives:
\[
\sum_{i=1}^{N}\sum_{k=1}^{N'} \gamma_{ik} \|\MM_i - \boldsymbol{\mu}_k\|_2^2 
\geq \frac{1}{N}\sum_{i=1}^{N}\|\MM_i - \boldsymbol{\mu}_{k(i)}\|_2^2
= \sum_{i=1}^{N}\sum_{k=1}^{N'} \gamma^*_{ik} \|\MM_i - \boldsymbol{\mu}_k\|_2^2.
\]

Thus $\gamma^*$ is optimal, and the squared WSD equals:
\[
\textsf{WSD}^2(\mathcal{P}, \mathcal{Q}) = \frac{1}{N}\sum_{k=1}^{N'}\sum_{\MM_i \in \mathcal{C}_k}\|\MM_i - \boldsymbol{\mu}_k\|_2^2.
\]

\textbf{Step 5: Equivalence to sketching objective.}
Rewriting using $\mathbb{1}_{[\SM_{k,i}\neq 0]} = \mathbb{1}_{[i \in \mathcal{C}_k]}$ and substituting $\boldsymbol{\mu}_k = \sum_{j=1}^{N} \SM_{k,j} \MM_j$:
\begin{align*}
\textsf{WSD}^2(\MM, \SM^\top\MM) 
&= \frac{1}{N} \sum_{k=1}^{N'} \sum_{i=1}^{N} \mathbb{1}_{[\SM_{k,i}\neq 0]} \cdot \left\| \MM_i - \boldsymbol{\mu}_k \right\|_2^2 \\
&= \frac{1}{N} \sum_{k=1}^{N'} \sum_{i=1}^{N} \mathbb{1}_{[\SM_{k,i}\neq 0]} \cdot \left\| \MM_i - \sum_{j=1}^{N} \SM_{k,j} \MM_j \right\|_2^2.
\end{align*}

Therefore, minimizing $\textsf{WSD}(\MM, \SM^\top\MM)$ is equivalent to the stated minimization problem, which corresponds precisely to the K-Means clustering objective in the feature space.
\end{proof}
\section{TAG Distillation Time and LLM4TAG training Time}
\label{appendix:real-time}

\subsection{TAG Distillation Time Comparison}

 Here we give a simple comparison between \algo{} and \ClustGDD in terms of distillation time, another efficient graph distillation method, at $r=0.05$ on {\em Cora} and {\em History}. We select \ClustGDD as the baseline because it achieves both high efficiency and accuracy, being orders of magnitude faster than other gradient-based or distribution-matching distillation methods while maintaining competitive performance.

Experimental results demonstrate the efficiency trade-offs of the \algo{} method: when utilizing an LLM for text synthesis in $\algo{}(\mathcal{S}')$, it consumes $191.67$ seconds and $291.24$ seconds on {\em Cora} and {\em History} datasets respectively, which is slower than $\algo{}(\XM')$. This overhead primarily stems from calling the LLM to generate $\mathcal{S}'$, which incurs high computational complexity. Meanwhile, both variants lag far behind \ClustGDD ($6.79$ seconds and $10.31$ seconds). However, \algo{} achieves higher accuracy and better interpretability than \ClustGDD, as demonstrated in previous experiments. Therefore, a time gap of the same order of magnitude or even one order of magnitude larger is acceptable.

\begin{table}[t]
\centering
\caption{Distillation running time comparison (seconds)}
\label{tab:time}
\footnotesize
\setlength{\tabcolsep}{4pt}
\begin{tabular}{lccc}
\toprule
\textbf{Dataset} & {$\algo{}(\XM')$} & {$\algo{}(\mathcal{S}')$} & {\ClustGDD} \\
\midrule
% \rowcolor{orange!10}
{\em Cora} & 56.0 & 191.7 &  6.79 \\
% \rowcolor{orange!10}
{History} & 83.1 & 291.2 & 10.31 \\
\bottomrule
\end{tabular}
\end{table}

\subsection{Time Cost of LLM4TAG Methods on $\G$ and $\G'$}

We compare the preprocessing and training times (in seconds) of three LLM4TAG methods,\GNNLLM~\cite{wullms}, \ENGINE~\cite{zhu2024efficient}, and \TAPE~\cite{he2023harnessing} on the original graph $\G$ and the distilled graph $\G'$ at distillation ratio $r=0.5$. \GNNLLM extracts LLM last-layer representation to train GNNs; \ENGINE fuses intermediate-layer LLM representations for GNN training; \TAPE prompts LLMs to generate predictions and natural language explanations for textual graph data, incurring substantially higher computational costs due to autoregressive generation.

\begin{table}[htbp]
\renewcommand{\arraystretch}{0.9}  
\centering
\footnotesize
\caption{Preprocessing time (seconds) of semi-supervised LLM4TAG methods at $r=0.5$ on $\G$ (Full Methods) and $\G'$ (\algo{}).}
\vspace{-2ex}
\label{tab:distillation_llm4g_preprocessing}
\addtolength{\tabcolsep}{-0.3em}  
\begin{tabular}{l c *{3}{c}} 
\toprule
\multirow{2}{*}{\textbf{Dataset}} & \multirow{2}{*}{\textbf{\makecell{Distillation\\ Method}}} & \multicolumn{3}{c}{\textbf{Baselines}} \\
\cmidrule(lr){3-5}  
& & \GNNLLM & \ENGINE & \TAPE \\
\midrule
\multirow{2}{*}{\em Cora}       
& Full Methods       & 22.5 & 111.0  &  2784.7   \\
& \algo{}               &  3.9  &  6.8 &  50.0  \\
\midrule
\multirow{2}{*}{\em History}    
& Full Methods       & 385.2 & 1743.1 & 32756.0 \\
& \algo{}         &  5.6 & 8.5 &  75.0  \\
\bottomrule
\end{tabular}
\end{table}

\textbf{Preprocessing Time Analysis.} 
The preprocessing time is substantially reduced on the distilled graph $\G'$ compared to the original graph $\G$ for all three methods. The reduction is most pronounced for \TAPE, which requires invoking LLMs to generate predictions and explanations for each node. On \textit{Cora}, \TAPE's preprocessing time drops from 2,784.7s to 50.0s (55.7$\times$ speedup); on the larger \textit{History} dataset, it decreases from 32,756.0s to 75.0s (436.7$\times$ speedup). This dramatic improvement stems from \algo{} reducing the graph size, thereby minimizing the number of expensive LLM API calls. \GNNLLM and \ENGINE also benefit from distillation, though with more modest reductions, as they only require forward passes through the LLM to obtain embeddings or intermediate representations, avoiding the costly autoregressive generation process. Due to the early stopping, \GNNLLM spends less time on the larger dataset {\em History}. 

\begin{table}[htbp]
\renewcommand{\arraystretch}{0.9}  
\centering
\footnotesize
\caption{Training time (seconds) of semi-supervised LLM4TAG methods at $r=0.5$ on $\G$ (Full Methods) and $\G'$ (\algo{}).}
\vspace{-2ex}
\label{tab:distillation_llm4g_training}
\addtolength{\tabcolsep}{-0.3em}  
\begin{tabular}{l c *{3}{c}} 
\toprule
\multirow{2}{*}{\textbf{Dataset}} & \multirow{2}{*}{\textbf{\makecell{Distillation\\ Method}}} & \multicolumn{3}{c}{\textbf{Baselines}} \\
\cmidrule(lr){3-5}  
& & \GNNLLM & \ENGINE & \TAPE \\
\midrule
\multirow{2}{*}{\em Cora}       
& Full Methods       & 4.2 &  10.1  & 7.1 \\
& \algo{}              &  3.3 & 10.0  &   4.3 \\
\midrule
\multirow{2}{*}{\em History}    
& Full Methods        & 3.5 & 121.3 & 36.4 \\
& \algo{}               &  3.2 &   86.4 &   35.0 \\
\bottomrule
\end{tabular}
\end{table}

\textbf{Training Time Analysis.} 
The training time reductions on $\G'$ are comparatively modest across all methods. For \GNNLLM and \TAPE, training times remain similar between $\G$ and $\G'$ because the downstream GNN training is already efficient even on the full graph. \ENGINE shows a more noticeable reduction on \textit{History} (from 121.3s to 86.4s), likely due to its architecture-specific fusion mechanism benefiting from the smaller distilled graph.

\begin{table*}[!ht]
\centering
\small
\setlength{\tabcolsep}{1.8pt}
\caption{Dataset statistics before/after distillation at $r=0.05$: nodes, edges, storage (MB), compression ratio, and accuracy.}
\label{tab:distillation_single_col}
\vspace{-2ex}
\begin{tabular}{l|c|rr|rrrr|r|r|r}
\toprule
\textbf{Dataset} & \textbf{Stage} & \textbf{Nodes} & \textbf{Edges} & 
\textbf{$\XM$(MB)} & \textbf{$\AM$ (MB)} & \textbf{$\YM$ (MB)} & \textbf{$\mathcal{S}$ (MB)} & 
\textbf{Total (MB)} & \textbf{Compression Ratio} & \textbf{Acc (\%)} \\
\midrule
\multirow{2}{*}{\em Cora}
    & Original  & 2708 & 5278 & 3.9668 & 0.1007 & 0.0207 & 2.3036 & 6.3918 & -- & 77.7$\pm$2.0  \\
    & Synthetic & 7 & 47 & 0.0103 & 0.0009 & 0.0001 & 0.0045 & 0.0158 & 403.39$\times$ & 77.5$\pm$0.6 \\
\midrule
\multirow{2}{*}{\em Citeseer}
    & Original  & 3186 & 8450 & 4.6670 & 0.1612 & 0.0243 & 6.2181 & 11.0706 & -- & 74.0$\pm$2.1  \\
    & Synthetic & 6 & 36 & 0.0088 & 0.0007 & 0.0000 & 0.0069 & 0.0164 & 673.20$\times$ & 75.9$\pm$0.5  \\
\midrule
\multirow{2}{*}{\em DBLP}
    & Original  & 14376 & 431326 & 21.0586 & 8.2269 & 0.1097 & 2.4339 & 31.8291 & -- & 77.9$\pm$0.5 \\
    & Synthetic & 4 & 16 & 0.0059 & 0.0003 & 0.0000 & 0.0023 & 0.0085 & 3739.87$\times$ & 77.2$\pm$0.2  \\
\midrule
\multirow{2}{*}{\em Computers}
    & Original  & 87229 & 721081 & 127.7769 & 13.7535 & 0.6655 & 42.7394 & 184.9353 & -- & 72.2$\pm$ 1.9 \\
    & Synthetic & 10 & 99 & 0.0146 & 0.0019 & 0.0001 & 0.0123 & 0.0289 & 6397.35$\times$ &  65.4$\pm$ 3.7 \\
\midrule
\multirow{2}{*}{\em Photo}
    & Original  & 48362 & 500939 & 70.8428 & 9.5547 & 0.3690 & 37.8577 & 118.6242 & -- & 70.0$\pm$ 0.9 \\
    & Synthetic & 12 & 139 & 0.0176 & 0.0027 & 0.0001 & 0.0140 & 0.0344 & 3451.25$\times$ & 64.6$\pm$ 1.4 \\
\midrule
\multirow{2}{*}{\em History}
    & Original  & 41551 & 358574 & 60.8657 & 6.8393 & 0.3170 & 57.3151 & 125.3371 & -- & 70.6$\pm$3.6 \\
    & Synthetic & 11 & 106 & 0.0161 & 0.0020 & 0.0001 & 0.0110 & 0.0292 & 4291.29$\times$ & 75.9$\pm$1.7 \\
\midrule
\multirow{2}{*}{WikiCS}
    & Original  & 11701 & 431726 & 17.1401 & 8.2345 & 0.0893 & 145.8948 & 171.3587 & -- & 75.9$\pm$ 1.2  \\
    & Synthetic & 10 & 100 & 0.0146 & 0.0019 & 0.0001 & 0.0062 & 0.0228 & 7505.11$\times$ & 76.6$\pm$ 0.3 \\
\bottomrule
\end{tabular}
\end{table*}

\section{Statistics of Original and Synthetic TAGs}
\label{appendix:stat}
Table~\ref{tab:distillation_single_col} summarizes the dataset statistics and performance before and after distillation at the reduction ratio $r=0.05$. We report the number of nodes and edges, storage usage (in MB) for attributes, adjacency matrices, labels, and text, as well as total storage consumption, compression ratio, and classification accuracy (Acc, \%). It can be observed that our distillation method achieves extremely high compression ratios across all seven datasets. The compression ratios range from 403.39$\times$ for the Cora dataset to 7505.11$\times$ for the WikiCS dataset, indicating a significant reduction in data volume. Specifically, the total storage of the original datasets, which ranges from several MB to nearly 200 MB, is compressed to only about 0.01–0.04 MB for the synthetic datasets. Meanwhile, the number of nodes is drastically reduced from thousands to tens (4–12 nodes for synthetic datasets), and the number of edges is reduced from thousands to hundreds (16–139 edges for synthetic datasets), realizing efficient compression of graph structure.
Despite the extreme compression of graph structure and data volume, the synthetic datasets largely preserve the predictive performance of the original graphs. For most datasets, the classification accuracy of the synthetic datasets is close to that of the original datasets with only minor fluctuations. Notably, the synthetic datasets of Citeseer, History, and WikiCS achieve comparable or slightly higher accuracy than their original counterparts, which demonstrates that our distillation method can effectively retain the critical structural, attribute, and semantic information required for downstream classification tasks.
For the Computers and Photo datasets, there is a slight drop in accuracy of the synthetic datasets compared to the original ones, but the accuracy remains at a reasonable level, which is acceptable given the extremely high compression ratio. In summary, the experimental results show that our distillation method can achieve efficient compression of graph datasets while maintaining good performance, verifying the effectiveness and practicality of the method.

\clearpage
% % \onecolumn
\section{Prompts and Results}
\label{app:prompts}
\subsection{Prompts for text synthesis in \algo{}}
First, we present the complete prompt templates used in the keywords-based LLM text synthesis in \algo{}, including the system initialization template, user task instruction template, and corresponding dataset background descriptions.

\stitle{System Prompt Template}
This template serves as the initial system instruction to define the role and global context for the researcher. It provides the basic background of the TAG dataset and lays the foundation for subsequent document summarization tasks.

\stitle{Dataset Background Descriptions}
To complement the placeholders $\text{'dataset\_name'}$ and $\text{'{sys\_dataset}'}$ in the system prompt template, detailed background descriptions of typical datasets used in this task are provided below. Taking $\text{SYS\_CORA}$ and $\text{SYS\_HISTORY}$ as examples, we clearly present the structure and category information of each dataset.

\stitle{User Prompt Template}
Building upon the system prompt, this user prompt template further provides specific task instructions, strict mandatory requirements, and standardized output formats for clustered document summarization. It clarifies the key constraints and operational details to ensure the validity of the generated outputs.

\begin{lstlisting}[style=prompttemplate]
System Prompt Template:
    """
        We have performed clustering on the Text-Attributed Graph dataset and extracted keywords for the texts in each cluster.
        You are a professional researcher tasked with summarizing clustered documents based on keywords. 
        Follow ALL instructions strictly.
        
        # Text-Attributed Graph dataset Background (Global Context)
        This task processes documents from the "{dataset_name}" dataset:
        {sys_dataset}
    """
\end{lstlisting}

\begin{lstlisting}[style=datasetstyle]
SYS_CORA:
    """
    It is an academic citation network in the field of machine learning, consisting of 2708 nodes (each node represents a machine learning-related paper), 10556 edges (each edge represents the citation relationship between papers), and all nodes are divided into 7 category labels, namely Rule_Learning, Neural_Networks, Case_Based, Genetic_Algorithms, Theory, Reinforcement_Learning, Probabilistic_Methods.
    """
\end{lstlisting}
\vspace{2cm}
\begin{lstlisting}[style=datasetstyle]
SYS_HISTORY:
    """
    It is a book recommendation network extracted from the Amazon-Books dataset, consisting of  41551 nodes (each node represents a book with the second-level label History),  503180 edges (each edge represents two books are frequently co-purchased or co-viewed), and all nodes are divided into 12 category labels (the three-level labels of the books): World, Americas, Asia, Military, Europe, Russia, Africa, Ancient Civilizations, Middle East, Historical Study & Educational Resources, Australia & Oceania, Arctic & Antarctica. 
    """
\end{lstlisting}

\begin{lstlisting}[style=prompttemplate]
User Prompt Template:
    """
        Task Instructions:
        Clustered documents' keywords (prioritize first 3-5 as core): {center_key}
        Possible documents' category: {center_label1}
        Strict Mandatory Requirements (VIOLATION = INVALID OUTPUT): 
        1. Token Limit: Total output (including [[LABEL]], [[SUMMARY]], spaces, punctuation) must NOT exceed {tok_lim} tokens.
        2. Output Format (EXACTLY as shown, NO extra lines/words/symbols/line breaks/markdown):
        [[LABEL]] <selected category name>
        [[SUMMARY]] <single short paragraph summary>
        3. [[LABEL]] Rule: 
           - Select ONLY ONE category from the given category ({center_label1}).
           - Output ONLY the category name (no additional text, punctuation, or explanations).
        4. [[SUMMARY]] Rules (ALL must be satisfied):
           - Length: Max {tok_lim} tokens (after [[SUMMARY]]:).
           - Content: 
            - Integrates the first 3-5 core keywords naturally into the summary (avoid keyword stacking); incorporate other keywords only if they fit the context without exceeding token limits;
            - Clearly and concisely summarizes the core theme/key findings/primary focus of the clustered documents (1-2 core points only, no trivial details);
           - Format: Single short paragraph (no bullet points, no line breaks, no lists, no markdown).
    
        Output NOW (strictly follow all rules above):
    """
\end{lstlisting}

\subsection{Keywords and Synthesis in \algo{}}

To elaborate on the specific results of keyword extraction and cluster summarization, we first analyze the {\em Cora} and {\em History} under the condition of $r=0.05$. We here present the keywords and corresponding cluster summarizations. 

Focusing on the class $\text{Neural Network}$ in the {\em Cora}, we extracted 256 keywords from its corresponding cluster, which are listed as follows. 
\begin{lstlisting}[style=keywordstyle]
Keywords: 
    """
    classifiers, neural, perceptrons, classifier, backpropagation, supervised, classification, features, recognition, feature, networks, neuron, learns, learning, neurons, regularization, algorithms, cortex, perceptron, adaptive, nodes, models, ensemble, clustering, prediction, generalization, handwritten, machines, computational, segmentation, detection, wavelet, overfitting, connectionist, inputs, computing, feedforward, multilayer, representations, net, predict, trained, sparse, algorithm, synaptic, modeling, optimization, network, brain, patterns, mlp, signals, learned, knowledge, layers, generative, model, receptive, gradient, processing, kernel, stochastic, implementation, ...
    """
\end{lstlisting}
We can see that the keywords are with high related to the class, shown the effectiveness of our keywords extraction. Based on the extracted keywords of the $\text{Neural Network}$ cluster, we generated three summarizations as synthetic text to capture the core themes of the cluster, as shown below.
\begin{lstlisting}[style=summarystyle]
Candidate 1: 
    """
    [[LABEL]] Neural_Networks
    [[SUMMARY]] The cluster focuses on neural networks, particularly classifiers and perceptrons, emphasizing supervised learning and backpropagation for tasks like classification and pattern recognition. Core themes include network architectures such as multilayer perceptrons (MLPs), regularization techniques to combat overfitting, and feature representation learning. Additional topics span gradient-based optimization, activation functions, generalization, and applications in vision, signal processing, and computational modeling inspired by biological neurons and cortical structures.
    """

Candidate 2:
    """
    [[LABEL]] Neural_Networks
    [[SUMMARY]] The clustered documents focus on neural networks, particularly exploring classifiers, neural architectures, and perceptrons in supervised learning contexts. Core themes include backpropagation for training multilayer networks, feature extraction for classification tasks, and strategies to improve generalization while mitigating overfitting. Additional topics span regularization, adaptive learning, and applications in pattern recognition, such as handwritten digit detection and signal processing, often leveraging feedforward and recurrent network models.
    """
Candidate 3: 
    """
    [[LABEL]] Neural_Networks
    [[SUMMARY]] This cluster focuses on neural networks, emphasizing classifiers, neural architectures, and perceptrons, with significant attention to backpropagation and supervised learning for classification tasks. Core themes include network models like multilayer perceptrons (MLPs), feature representation, generalization, and techniques to mitigate overfitting through regularization. Additional topics span learning algorithms, adaptive systems, and applications in pattern recognition, image processing, and computational neuroscience, often leveraging feedforward and recurrent structures, activation functions, and gradient-based optimization to improve prediction accuracy and model robustness.
    """
\end{lstlisting}

Similarly, for $\text{Ancient Civilizations} $ in {\em History}, the keywords and corresponding summarizations are 
\begin{lstlisting}[style=keywordstyle]
 Keywords: 
     """
     antiquity, archaeologists, antiquities, archeology, egyptology, archaeology, archaeological, ancient, archaeologist, egyptologist, egyptians, mesopotamia, tombs, historia, histories, historians, archaelogy, monuments, archeologist, tomb, inscriptions, byzantium, babylon, egyptian, romans, byzantine, greeks, anthropology, excavated, antique, coptic, egypt, egypts, grecoroman, history, civilization, bookdescription, hammurabi, museum, paganism, rome, bce, books, chronicle, greek, johannes, depictions, temple, roman, book, emperors, tutankhamun, paintings, cairo, mysteries, turks, statuette, deities, myths, oetempledestructiona, artifact, goddesses, literary, empire, autobiography,...
     """
\end{lstlisting} 

\begin{lstlisting}[style=summarystyle]
Candidate 1: 
    """
    [[LABEL]] Ancient Civilizations  
    [[SUMMARY]] This cluster centers on antiquity, archaeologists, and antiquities, with a strong focus on archaeology and egyptology. The documents explore ancient civilizations such as Egypt, Mesopotamia, and the Greco-Roman world, examining archaeological findings, tombs, inscriptions, and artifacts. Key themes include the study of ancient cultures through excavated materials, historical interpretations by renowned scholars, and insights into religious practices, monuments, and daily life in antiquity.
    """

Candidate 2: 
    """
    [[LABEL]] Ancient Civilizations  
    [[SUMMARY]] The cluster centers on antiquity, archaeologists, and antiquities, with a strong focus on archaeology and egyptology. Documents explore ancient civilizations like Egypt, Mesopotamia, and Rome through archaeological findings, tombs, inscriptions, and artifacts, highlighting scholarly research on deities, emperors, and historical narratives.
    """

Candidate 3: 
    """
    [[LABEL]] Ancient Civilizations  
    [[SUMMARY]] The cluster centers on antiquity, archaeologists, and antiquities, with a strong focus on archaeology and egyptology. Documents explore ancient civilizations like Egypt, Mesopotamia, and Rome through archaeological findings, tombs, inscriptions, and artifacts, highlighting scholarly research on deities, emperors, and historical narratives.
    """
\end{lstlisting}

The reason why the summaries for {\em Cora} and {\em History} datasets differ in length is that the optimal output token limit varies for different downstream classification tasks.

\subsection{Synthetic Texts from \algo{}, \DaLLME Comparison}

We here give a comparison between texts synthesized by \algo{} and \DaLLME. A qualitative analysis reveals the superiority of \algo{}, which adopts a keyword-guided LLM framework for cluster summarization, over \DaLLME. \DaLLME directly maps cluster representations to natural language without explicit semantic guidance. As demonstrated in the {\em Cora} and {\em History} examples, \algo{} generates domain-consistent summaries with strong topic alignment and fidelity to domain knowledge by grounding generation in cluster keywords, effectively preserving fine-grained terminology and capturing the full semantic scope of each document cluster. \DaLLME, even with text-pair fine-tuning, is constrained by the inherent information bottlenecks of its V2T paradigm, leading to outputs that are relatively less comprehensive. 

Here are samples of {\em Cora}:
\begin{lstlisting}[style=summarystyle]
STAD:
    """
    [[LABEL]] Neural_Networks
    [[SUMMARY]] The cluster focuses on neural networks, particularly classifiers and perceptrons, emphasizing supervised learning and backpropagation for tasks like classification and pattern recognition. Core themes include network architectures such as multilayer perceptrons (MLPs), regularization techniques to combat overfitting, and feature representation learning. Additional topics span gradient-based optimization, activation functions, generalization, and applications in vision, signal processing, and computational modeling inspired by biological neurons and cortical structures.
    """

DaLLME: 
    """
    "Analysis of Neural Networks," : In this paper we present a method for evaluating the performance of neural networks in a variety of contexts. The method is based on the idea that neural networks can be trained to perform a given task in a given time frame
    """

\end{lstlisting}

Here are samples of {\em History}
\begin{lstlisting}[style=summarystyle]
STAD:
    [[LABEL]] Ancient Civilizations  
    [[SUMMARY]] This cluster centers on antiquity, archaeologists, and antiquities, with a strong focus on archaeology and egyptology. The documents explore ancient civilizations such as Egypt, Mesopotamia, and the Greco-Roman world, examining archaeological findings, tombs, inscriptions, and artifacts. Key themes include the study of ancient cultures through excavated materials, historical interpretations by renowned scholars, and insights into religious practices, monuments, and daily life in antiquity.

DaLLME:
    feature node. BookDescription: &ldquo;This book is a must-read for anyone interested in the history of the ancient world.&rdquo;&mdash;Journal of American History&ldquo;&mdash;Journ

\end{lstlisting}

\subsection{Details of In-Context Learning based on \algo{} }

The \algo{} based ICL prompt generation takes the following steps. First, it loads test node attributes, synthetic attribute and labels from \algo{}. Unique categories of the current dataset are extracted from the loaded labels, forming the actual category list to fill the placeholder in the predefined templates. Then we compute Euclidean distances between each test node's attribute vector and the  synthetic attribute vectors. For each category, the algorithm further conducts statistical analysis on key distance metrics, including the average distance between the test node and all synthetic nodes in the category, the closest distance to any synthetic node in the category, and the index of the closest synthetic node.  \algo{} based ICL is effective because it naturally fuses rich semantic priors with structural and feature similarity from the distilled dataset, which preserves the core information of the original graph. By incorporating category-aware distance guidance and semantic context into prompts, ICL enables the model to directly reason over compact yet representative knowledge without fine-tuning, leading to strong generalization and accurate node classification.

We here give the prompts of two training-free LLMs node classification Direct and {\em neighbor summary} (NS) , with and without the enhance of {\em in-context learning} (ICL) based on our \algo{}. We take {\em Cora} as an samples.

\stitle{System Prompt Template}
\begin{lstlisting}[style=prompttemplate]
    'You are an accurate text graph node classifier. Your task is to categorize text-described nodes into specified categories. Here is the scenario:\n'
\end{lstlisting}

\stitle{Usr Prompt Template}
For Direct and Direct+\algo{}, the user prompts are as follow, 
\begin{lstlisting}[style=prompttemplate]
Direct: 
    """
    {input text}
    Question: Which of the following sub-categories of AI does this paper belong to? Here are the 7 categories: Rule_Learning, Neural_Networks, Case_Based, Genetic_Algorithms, Theory, Reinforcement_Learning, Probabilistic_Methods. Reply only one category that you think this paper might belong to. Only reply the category phrase without any other explanation words.
    
    Answer: 
    """

Diectly + STAD: 
    """
    {input text}
    Question: Which of the following sub-categories of AI does this paper belong to? Here are the 7 categories: Rule_Learning, Neural_Networks, Case_Based, Genetic_Algorithms, Theory, Reinforcement_Learning, Probabilistic_Methods. Reply only one category that you think this paper might belong to. Only reply the category phrase without any other explanation words.
    
    Answer: 
    Note: The closer the distance, the higher the probability of belonging to this class.
    
    Context information by category:
    - Rule_Learning: Average distance = {}, Closest distance = {}, Closest sample info: {}
    - Neural_Networks: Average distance = {}, Closest distance = {}, Closest sample info: {}
    - Case_Based: Average distance = {}, Closest distance = {}, Closest sample info: {}
    - Genetic_Algorithms: Average distance = {}, Closest distance = {}, Closest sample info: {}
    - Theory: Average distance = {}, Closest distance = {}, Closest sample info: {}
    - Reinforcement_Learning: Average distance = {}, Closest distance = {}, Closest sample info: {}
    - Probabilistic_Methods: Average distance = {}, Closest distance = {}, Closest sample info: {}
    
    Answer: 
    """
\end{lstlisting}

We show the pipeline of NS. Firstly, we give the prompt to for neighbor context summarization, 
\begin{lstlisting}[style=prompttemplate]
Neighbor Context Summarization: 
    """
    The following list records some paper related to the current one, with relationship being citation.
    Neighbor's text information: {neighbor_1_text}
    Neighbor's text information: {neighbor_2_text}
    Neighbor's text information: {neighbor_3_text}
    ...
    Please summarize the information above with a short paragraph, find some common points which can reflect the category of this paper.
    """
\end{lstlisting}

Then the query for node classification based on the neighbor context is 
\begin{lstlisting}[style=prompttemplate]
NS: 
    """
    {input text}
    {neighbor context summary}
    Here I give you the content of the node itself and the summary information of its 1st-order neighbors. 
    The relation between the node and its neighbors is 'citation'. Question: Based on these inforamtion, 
    Which of the following sub-categories of AI does this paper(this node) belong to? Here are the 7 categories: Rule_Learning, Neural_Networks, Case_Based, Genetic_Algorithms, Theory, Reinforcement_Learning, Probabilistic_Methods. 
    Reply only one category that you think this paper might belong to. Only reply the category name without any other words.\n\nAnswer:
    """
\end{lstlisting}

Similar to Direct + \algo{}, the query for NS + \algo{} is 
\begin{lstlisting}[style=prompttemplate]
NS + STAD: 
    """
    {input text}
    {neighbor context summary}
    Here I give you the content of the node itself and the summary information of its 1st-order neighbors. 
    The relation between the node and its neighbors is 'citation'. Question: Based on these inforamtion, 
    Which of the following sub-categories of AI does this paper(this node) belong to? Here are the 7 categories: Rule_Learning, Neural_Networks, Case_Based, Genetic_Algorithms, Theory, Reinforcement_Learning, Probabilistic_Methods. 
    Reply only one category that you think this paper might belong to. Only reply the category name without any other words.
    
    Answer:
    Note: The closer the distance, the higher the probability of belonging to this class.
    
    Context information by category:
    - Rule_Learning: Average distance = {}, Closest distance = {}, Closest sample info: {}
    - Neural_Networks: Average distance = {}, Closest distance = {}, Closest sample info: {}
    - Case_Based: Average distance = {}, Closest distance = {}, Closest sample info: {}
    - Genetic_Algorithms: Average distance = {}, Closest distance = {}, Closest sample info: {}
    - Theory: Average distance = {}, Closest distance = {}, Closest sample info: {}
    - Reinforcement_Learning: Average distance = {}, Closest distance = {}, Closest sample info: {}
    - Probabilistic_Methods: Average distance = {}, Closest distance = {}, Closest sample info: {}
    
    Answer: 
    """
\end{lstlisting}

\end{document}